\pdfoutput=1
\documentclass[11pt]{article}
\usepackage{graphicx} % Required for inserting images
\usepackage{diagbox}
\usepackage{authblk}
\makeatletter
\renewcommand\AB@authnote[1]{\textsuperscript{#1}\hspace{5pt}}

\newcommand{\Opt}{\text{Opt}}
\makeatother
\usepackage{hyperref}
\usepackage{algorithm}
\usepackage{algpseudocode}
\usepackage{notation}
\usepackage{arxiv-2}
\usepackage{amssymb}
\usepackage{mathrsfs}
\usepackage{float}
\usepackage{setspace}
\usepackage{lmodern}        % Latin Modern 是 Computer Modern 的扩展字体（等宽支持更好）
\renewcommand{\tilde}[1]{\widetilde{#1}}

\renewcommand{\hat}[1]{\widehat{#1}}
\newcommand{\dia}{\diamond}
\newcommand{\cl}{\ell}

\newcommand{\pll}{\kern 0.3em/\kern -0.9em /\kern 0.3em}
\newcommand{\sh}{\sharp}

\title{\normalfont Perturbing the Derivative: Doubly Wild Refitting for Model-Free Evaluation of Opaque Machine Learning Predictors
}
\begin{centering}
     \author[1,3]{Haichen Hu\thanks{{Email: \texttt{huhc@mit.edu}}}} 
     \author[2,3]{David Simchi-Levi\thanks{{Email: \texttt{dslevi@mit.edu}}}}
\end{centering}
\affil[1]{Center for Computational Science and Engineering, MIT} 
\affil[2]{Institute for Data, Systems, and Society, MIT} 
\affil[3]{Department of Civil and Environmental Engineering, MIT}

\date{}
\begin{document}

\maketitle
\begin{abstract}
    We study the problem of excess risk evaluation for empirical risk minimization (ERM) under convex losses.  We show that by leveraging the idea of wild refitting \citep{wainwright2025wild}, one can upper bound the excess risk through the so-called “wild optimism,” without relying on the global structure of the underlying function class but only assuming black-box access to the training algorithm and a single dataset. We begin by generating two sets of artificially modified pseudo-outcomes created by stochastically perturbing the derivatives with carefully chosen scaling. Using these pseudo-labeled datasets, we refit the black-box procedure twice to obtain two wild predictors and derive an efficient excess-risk upper bound under the fixed design setting. Requiring no prior knowledge of the complexity of the underlying function class, our method is essentially model-free and holds significant promise for theoretically evaluating modern opaque deep neural networks and generative models, where traditional learning theory could be infeasible due to the extreme complexity of the hypothesis class.

\end{abstract}
\vspace{-0.7cm}
\noindent\rule{\textwidth}{1pt} 
\vspace{-2.6em}
\section{Introduction}
Black-box machine learning and artificial intelligence algorithms—such as deep neural networks, large language models (LLMs), and emerging forms of artificial general intelligence (AGI)—are profoundly reshaping modern science \citep{jumper2021highly}, industry \citep{saka2024gpt}, business \citep{alim2025beyond,huang2025orlm}, healthcare \citep{wang2020should}, and education \citep{fuchs2023exploring}. From the perspective of machine learning applications, such as in operations management, people have been increasingly adopting highly complex predictive algorithms—most notably deep neural networks—to estimate and forecast a wide range of structural functions, including demand functions \citep{safonov2024neural}, customer choice models \citep{gabel2022product,chen2023machine}, logistics and transportation costs \citep{tsolaki2023utilizing}, supply delay distributions, and inventory dynamics \citep{boute2022deep}. These opaque systems are becoming increasingly central to real-world decision-making; empirically evaluating their effectiveness, robustness, and interpretability has become both essential and widely studied \citep{linardatos2020explainable,2021measuring,ye2025document,hendrycksmeasuring}. 

At the same time, in statistical machine learning and its applications in operations research, including dynamic pricing \citep{wang2025dynamic}, inventory control \citep{ding2024feature}, and online assortment optimization \citep{bastani2022learning}—researchers increasingly rely on the theoretical interpretability and statistical guarantees of learning algorithms. These theoretical insights provide rigorous control over risk and generalization power, forming the cornerstone of data-driven decision-making. Such guarantees typically depend on known structural properties of the underlying model class, often characterized by measures such as Rademacher complexity \citep{bartlett2005local}, VC dimension \citep{vapnik2015uniform}, and eluder dimension \citep{russo2013eluder}.

To evaluate ML models, researchers also use hold-out methods or cross-validation based on sample splitting \citep{reitermanova2010data, berrar2019cross}, where the trained model is tested on a separate dataset. However, such approaches can be data-inefficient, as they discard a portion of the available samples from the training process for testing. This limitation becomes particularly pronounced when evaluating large-scale AI models, which typically require a tremendous amount of data for reliable assessment \citep{ivanova2025towards}. Consequently, conventional sample-splitting methods may be unsuitable for evaluating modern AI models with billions of parameters \citep{achiam2023gpt, team2024gemma, grattafiori2024llama}.

Consequently, a significant gap remains between empirical and theoretical research regarding the evaluation of modern complex ML and AI models, which naturally leads to the following question:

\begin{center}
\emph{Can we provide a rigorous characterization of the risk regarding complicated machine learning training algorithms with only a single dataset and black-box access to the procedure?}
\end{center}

In this paper, we study the general Empirical Risk Minimization (ERM) framework under convex losses and develop an efficient algorithm to evaluate its excess risk using only black-box access to a single dataset, without sample splitting. To this end, we propose a novel doubly wild refitting procedure that enables a rigorous characterization of the excess risk for any opaque machine learning procedure. Unlike previous work on wild refitting \citep{wainwright2025wild}, which was limited to regression with symmetric noise, our approach handles a much broader range of applications and subsumes \citet{wainwright2025wild} as an special case.

\section{Related Works}
Our work is closely connected to two streams of research: statistical machine learning and sample-splitting evaluation.

Statistical learning \citep{vapnik2013nature} has long served as a building block in the theoretical analysis of machine learning algorithms, and the most popular form of machine learning algorithms is Empirical Risk Minimization (ERM) \citep{vapnik1991principles}. A central and active line in statistical learning focuses on understanding the \emph{excess risk} or generalization error of learning algorithms. Classical approaches rely on empirical process theory, with complexity measures such as the VC dimension \citep{vapnik2015uniform, floyd1995sample}, covering numbers \citep{van2000empirical}, Rademacher complexity \citep{massart2007concentration, bartlett2005local}, and eigendecay rate \citep{goel2017eigenvalue,hu2025contextual}. Despite these advances, bounding these metrics fundamentally depends on the structure of the underlying model class. When the hypothesis space is extremely complicated, these methods fail to yield meaningful excess risk guarantees. Recently, \citet{wainwright2025wild} proposed the idea of wild refitting to evaluate the mean-square risk by retraining a new predictor on an artificially constructed dataset. We characterize the essential principles behind and develop a substantially broader algorithmic framework, subsuming \citet{wainwright2025wild} as a special case.

In statistics, the quality of procedures is often evaluated through hold-out or sample-splitting methods, where the dataset is partitioned into training and evaluation subsets \citep{reitermanova2010data,dobbin2011optimally}. While effective, these approaches are inherently data-inefficient, as they withhold valuable samples from the training process. Related techniques, such as cross-fitting \citep{berrar2019cross,refaeilzadeh2009cross, gorriz2024k}, attempt to address evaluation robustness but impose a significant computational burden due to repeated retraining, while still suffering from the limitations of subset-based training. In contrast, our proposed method eliminates the need for sample splitting or hold-out sets entirely, thereby maximizing the utilization of existing data.

\paragraph{Paper Structure:}
The remainder of the paper is organized as follows. Section~\ref{sec:model} introduces our ERM setup, including the assumptions regarding the loss function, along with key definitions and properties of the excess risk. Section~\ref{sec:doubly_wild_refitting} then presents our doubly wild refitting algorithm in full detail. Finally, Section~\ref{sec:statistical_guarantees} establishes the corresponding theoretical excess risk guarantees.
\paragraph{Notations:} We denote $[n]=\{1,\dots,n\}$ and write $x_{1:n}=(x_1,\dots,x_n)$. Expectations with respect to a random variable $Y$ and the conditional law of $Y\mid X$ are denoted by $\EE_Y$ and $\EE_{Y\mid X}$, respectively. For a convex function $\phi$, $D_\phi$ denotes the associated Bregman divergence. Given an algorithm $\cA$ and a dataset $\cD$, $\cA(\cD)$ denotes the predictor trained on $\cD$. Independence between random variables $U$ and $Z$ is written as $U\perp Z$. For any $g:\cX\to\RR$, its empirical $L_2$ norm is denoted as $\|g\|_n=(\sum_{i=1}^ng(x_i)^2/n)^{1/2}.$ Similarly, for any vector $\beta\in\RR^n$, $\|\beta\|_n=(\frac{1}{n}\sum_{i=1}^{n}\beta_i^2)^{1/2}$ is the normalized $2$-norm.
For a bivariate function $f(x,y)$, $f'_i(x,y)$ denotes the partial derivative with respect to its $i$-th argument.

\section{Model: ERM with Convex Losses}\label{sec:model}
In this section, we formally introduce the \emph{Empirical Risk Minimization (ERM)} problem that serves as the foundational model of our study. Specifically, we consider a general ERM setting with an input space $\cX$ and an scalar output (or prediction) space $\cY=\RR$. 
A predictor is defined as a mapping $f: \cX \to \cY$, and the dataset is denoted by $\cD = \{(x_i, y_i)\}_{i=1}^n$.

In the \emph{fixed-design} setting, the covariates $\{x_i\}_{i=1}^n$ are treated as fixed, and $y_i$ is sampled independently from some unknown conditional distribution $\PP(\cdot|x_i)$. 
whereas in the \emph{random-design} setting, both $(x_i, y_i)$ are random. 
In this paper, \textbf{we focus on fixed-design}.
The population-level objective in fixed-design is given by
\[
f^* \in \argmin_{f} 
\left\{ 
    \frac{1}{n}\sum_{i=1}^{n} 
    \EE_{y_i}[\cl(f(x_i), y_i) \mid x_i] 
\right\}.
\]
The empirical counterpart, corresponding to the ERM estimator, is defined as
\[
\hat{f} \in \argmin_{f \in \cF} 
\left\{ 
    \frac{1}{n}\sum_{i=1}^{n} 
    \cl(f(x_i), y_i) 
\right\},
\]
where $\cl: \RR \times \RR \to [0, +\infty)$ denotes the loss function, 
and $\cF$ represents the function class over which optimization is performed. 
Throughout, we assume that $\cF$ is a \emph{convex Banach function class}. To facilitate our theoretical analysis, we impose the following assumption on the loss function.
\begin{assumption}\label{ass:loss_function}
The loss function $\cl$ has the following properties:
\begin{itemize}
    \item Given any $y\in\cY$, the function $\cl(z,y)$ is $\beta$-smooth and $\alpha$-strongly convex with respect to $z$.
\item $\forall z\in\RR$,  $-\cl_1'(z,y)$ is continuous and coercive with respect to $y$, $\lim_{\|y\|\rightarrow\infty}\frac{\cl_1'(z,y)y}{|y|}=-\infty$.
\item $\forall z\in\RR$,  $-\cl_1'(z,y)$ is monotone with respect to $y$, i.e., $(\cl_1'(z,y_1)-\cl_1'(z,y_2))(y_1-y_2)\le 0$.
\end{itemize}
\end{assumption}
This assumption is satisfied by a broad class of commonly used loss functions, including the squared loss, the cross entropy loss with clipping, and most $L_2$ regularized convex losses.
\iffalse
\begin{example}
    The squared loss $\cl(u,y)=\frac{1}{2}\|u-y\|_2^2$ satisfies Assumption \ref{ass:loss_function}.
\end{example}
\begin{example}
For exponential family with parameter $\theta$, $p(y|\theta)=h(y)\exp\cbr{\inner{\theta}{T(y)}-A(\theta)}$, where $T(y)$ is sufficient statistics and $A(\theta)=\log\int h(y)\exp\cbr{\inner{\theta}{T(y)}}dy$ is the normalization factor. Therefore, $-\log(p(y|\theta))=-\log h(y)-\inner{\theta}{T(y)}+A(\theta)$. The regularized negative log-likelohood loss function as $\cl(z,y)=A(z)-z^Ty+\frac{\alpha}{2}\|z\|_2^2$ satisfies Assumption \ref{ass:loss_function}.
\end{example}
\begin{example}
    In convex quadratic programming, the loss function is $\cl(x,y)=(x-y)^TA(x-y)+b^T(x-y)$ where $A$ is positive definite. This loss function satisfies Assumption \ref{ass:loss_function}.
\end{example}
\fi
By the first order optimality condition in variational analysis \citep{rockafellar1998variational}, we have the following proposition. 
\begin{proposition}\label{prop:gateaux1}
    For any $x\in\cX$, $f^*(x)$ minimizes the conditional expectation $\EE_{y}[\cl(f(x),y)|x]$. Then, the directional Gâteaux derivative at any measurable function  is zero, i.e.,
    \begin{align}\label{equa:gateaux1}
\EE_{y}[\cl_1'(f^*(x),y)h(x)|x]=0,\ \forall h\Rightarrow\EE_{y}[\cl_1'(f^*(x),y)|x]=0.
\end{align}
\end{proposition}
Therefore, given any $x$, $\cl_1'(f^*(x),y)$ is a zero-mean random variable. Naturally, in this paper, we assume the following sub-Gaussian tail property about it.
\begin{assumption}\label{ass:noise}
    Fix any $x\in\cX$, $\cl_1'(f^*(x),y)$ is $\sigma^2$ sub-Gaussian.
\end{assumption}
Similarly, since $\hat{f}$ minimizes the empirical risk, again by the first order condition about the Gâteaux derivative, we have the following proposition similar to Proposition \ref{prop:gateaux1}.
\begin{proposition}\label{prop:gateaux2}
For the empirical risk minimization procedure, if $\cF$ is convex, we have that
\begin{align}\label{equa:gateaux2}
\frac{1}{n}\sum_{i=1}^{n}\cl_1'(\hat{f}(x_i),y_i)(f(x_i)-\hat{f}(x_i))=0,\ \forall f\in\cF.
\end{align}
\end{proposition}
We now elaborate on the excess risk, which serves as the performance metric for the generalization power of any algorithm in statistical learning. In the fixed design setting, the \emph{excess risk} is defined as
\[
\cE_{fix}(\hat{f}):=\frac{1}{n}\sum_{i=1}^{n}\EE_{y'_i}[\cl(\hat{f}(x_i),y'_i)-\cl(f^*(x_i),y'_i)],
\]
where $\cbr{y_i'}_{i=1}^{n}$ are new test data points that are drawn from the same distribution and are independent of the training dataset. For the square loss, the corresponding excess risk is $\|\hat{f}-f^*\|_n^2$. Under the ERM procedure, we can also define the \emph{empirical excess risk} as
\[
\bar{\cE}_{\cD}(\hat{f}):=\frac{1}{n}\sum_{i=1}^{n}\cl(\hat{f}(x_i),y_i)-\frac{1}{n}\sum_{i=1}^{n}\cl(f^*(x_i),y_i).
\]
Now, to compare the difference between $\cE_{fix}(\hat{f})$ and $\bar{\cE}_{\cD}(\hat{f})$, we use the Bregman representation to obtain.
\begin{align}\label{equa:bregman}
    \cl(\hat{f}(x_i),y_i)=\cl(f^*(x_i),y_i)+\cl_1'(f^*(x_i),y_i)(\hat{f}(x_i)-f^*(x_i))+D_{\cl(\cdot,y_i)}(\hat{f}(x_i),f^*(x_i)).
\end{align}
Given equality \ref{equa:bregman}, Proposition \ref{prop:gateaux1} and \ref{prop:gateaux2}, we upper bound the excess risk $\cE_{fix}(\hat{f})$ by its empirical counterpart $\bar{\cE}_{\cD}(\hat{f})$ plus a stochastic term, as presented in Proposition \ref{prop:bound_excess_by_empirical}. The proof is deferred to Appendix \ref{app:proofs_sec:model}.
\begin{proposition}\label{prop:bound_excess_by_empirical} We have that
    \[
    \cE_{fix}(\hat{f})\le \frac{\beta}{\alpha}\bar{\cE}_{\cD}(\hat{f})+\frac{\beta}{\alpha}\frac{1}{n}\sum_{i=1}^{n}\cl_1'(f^*(x_i),y_i)(\hat{f}(x_i)-f^*(x_i)).
    \]
\end{proposition}
By the definition of $\bar{\cE}_{\cD}(\hat{f})$, we can upper bound it by the empirical training error $\frac{1}{n}\sum_{i=1}^n \cl(\hat{f}(x_i),y_i)$. In well-specified cases, we have $\bar{\cE}_{\cD}(\hat{f})\le 0$. Consequently, when combined with Proposition~\ref{prop:bound_excess_by_empirical}, controlling the excess risk $\cE_{fix}(\hat{f})$ reduces to bounding the quantity referred to as \emph{true optimism}.
\[
\Opt^*(\hat{f}):=\frac{1}{n}\sum_{i=1}^n\cl_1'(f^*(x_i),y_i)
(\hat{f}(x_i)-f^*(x_i)).
\]
This term represents the central analytical object we aim to bound throughout the paper.
\section{Doubly Wild Refitting via Perturbing the Derivatives}\label{sec:doubly_wild_refitting}
In this section, we introduce our algorithm, \emph{Doubly Wild Refitting}, which provides a function-class-free approach to bounding the excess risk. 
The central idea of wild refitting is to artificially construct new datasets by applying carefully designed perturbations to the predicted values via a recentering pilot predictor $\tilde{f}$. 
The model is then retrained on these perturbed datasets, and the resulting refitted models are used to extract statistical information that enables a bound on the excess risk.

Previous work on wild refitting \citep{wainwright2025wild} focuses on prediction problems with symmetric noise distributions. In this simpler setting, a single perturbed dataset and one model refit suffice to carry out the analysis. In contrast, the non-symmetric noise setting in this paper requires constructing two perturbed datasets that are coupled together and refitting the model on each of them separately. 

The double refitting procedure is crucial for enabling randomized symmetrization. Under non-symmetric noise, perturbations in one direction do not provide sufficient information about the behavior of the predictor in the opposite direction. As a result, explicit perturbations in both directions are required.

Our algorithm is divided into two sub-routines: perturbation construction and model refitting.
\paragraph{Perturbation Construction:} After training $\hat{f}=\Alg(\cD)$, with a recentering pilot predictor $\Tilde{f}$, we compute the derivatives $\tilde{g}_i=\cl_1'(\tilde{f}(x_i),y_i),\ i\in[n]$. 
After that, we construct a sequence of i.i.d. Rademacher random variables $\cbr{\varepsilon_i}_{i=1}^{n}$. 
Then, for two scaling numbers $\rho_1,\rho_2>0$, we construct the wild responses $\cbr{y_i^\dia}_{i=1}^{n}$ and $\{y_i^\sh\}_{i=1}^{n}$. In this paper, we offer two principled perturbation schemes with the equivalent efficiency: 
\begin{itemize}
    \item[(I)] $\forall i\in[n]$, $\cl_1'(\hat{f}(x_i),y_i^\dia)=\cl_1'(\hat{f}(x_i),y_i)-2\rho_1\varepsilon_i\tilde{g}_i,\ \cl_1'(\hat{f}(x_i),y_i^\sh)=\cl_1'(\hat{f}(x_i),y_i)+2\rho_2\varepsilon_i\tilde{g}_i.$
    \item[(II)] $\forall i\in[n]$, $\cl_1'(\hat{f}(x_i),y_i^\dia)=-2\rho_1\varepsilon_i\tilde{g}_i,\ \cl_1'(\hat{f}(x_i),y_i^\sh)=2\rho_2\varepsilon_i\tilde{g}_i.$
\end{itemize}
Unlike \citet{wainwright2025wild}, which perturbs predicted responses directly, our approach perturbs the derivative of the loss function. This choice deliberately alters the local geometry of the loss landscape while preserving its global structure. As a result, derivative perturbations provide fine-grained control over the growth of the loss around $\hat{f}$, which is essential for establishing non-asymptotic empirical process bounds.

As a concrete example, consider the squared loss $\ell(f(x),y)=(f(x)-y)^2$ and perturbation scheme~(II). In this case, the resulting perturbed response takes the form $y_i^\dia = \hat{f}(x_i) + 2\rho\,\varepsilon_i(y_i-\tilde{f}(x_i)),$
which coincides exactly with the wild perturbation proposed in \citet{wainwright2025wild}. Consequently, our doubly wild refitting framework generalizes the method of \citet{wainwright2025wild}, which is recovered as a special case. More generally, if the loss function $\cl$ itself is a Bregman divergence $D_{\phi}$ of some convex potential $\phi$, then utilizing the property of the Fenchel conjugate, perturbation scheme (II) can be equivalently written as
\[
y_i^\dia=(\phi^*)'(\phi(\hat{f}(x_i))-2\rho\varepsilon_i\tilde{w}_i),\ \tilde{w}_i=y_i-\tilde{f}(x_i),\ y_i^\sh=(\phi^*)'(\phi(\hat{f}(x_i))+2\rho\varepsilon_i\tilde{w}_i),\ \tilde{w}_i=y_i-\tilde{f}(x_i), i\in[n],
\]
which can be efficiently carried out via computing the derivative of the conjugate $\phi^*$ of $\phi$.

To show that our perturbation scheme is well defined, we establish the following proposition, which guarantees that the doubly wild refitting procedure is both well defined and computationally tractable. The theoretical justification relies on the Browder–Minty theorem \citep{browder1967existence,minty1962monotone} from functional analysis in Appendix~\ref{app:useful_math_tools}.

\begin{proposition}\label{prop:wild_refit_well_defined}
    Under Assumption \ref{ass:loss_function}, the doubly wild refitting procedure is well-defined. For any $1\le i\le n$, any scaling $\rho_1,\rho_2$, and any realization of $\varepsilon_i$, if we use scheme \text{I)}, then
    $\exists$ $y_i^\dia$ and $y_i^\sh$ s.t.
    \[
    \cl_1'(\hat{f}(x_i),y_i^\dia)=\cl_1'(\hat{f}(x_i),y_i)-2\rho_1\varepsilon_i\tilde{g}_i,\ 
    \cl_1'(\hat{f}(x_i),y_i^\sh)=\cl_1'(\hat{f}(x_i),y_i)+2\rho_2\varepsilon_i\tilde{g}_i.
    \]
    The same result holds for scheme \text{II)}. If $
    \cl(u,y)$ is strictly monotone with respect to $y$, $y_i^\dia$ and $ y_i^\sh$ are unique.
\end{proposition}

\paragraph{Model Refitting:} After applying the perturbation procedure, we construct two artificial datasets $\cD^\dia :=\cbr{(x_i, y_i^\dia)}_{i=1}^n
\text{and}\ \cD^\sharp := \{(x_i, y_i^\sharp)\}_{i=1}^n.$
We then retrain the model on $\cD^\dagger$ and $\cD^\sharp$, yielding the corresponding wild predictors $f_{\rho_1}^\dagger$ and $f_{\rho_2}^\sharp$, respectively.

Combining the perturbation procedure and the refitting procedure, we summarize the doubly wild refitting algorithm as pseudo-code in Algorithm~\ref{alg:wild-refitting}.
\begin{algorithm}
\begin{algorithmic}[1]
\caption{Doubly Wild Refitting with Convex Loss}\label{alg:wild-refitting}
\Require Procedure $\Alg$, dataset $\cD_0=\cbr{(x_i,y_i)}_{i=1}^{n}$, noise scales $\rho_1, \rho_2>0$, loss function $\ell$, and recentering pilot predictor $\tilde{f}$.
        \State Apply algorithm $\Alg$ on the training dataset. Get predictor $$\hat{f}=\Alg(\cbr{(x_i,y_i)}_{i=1}^{n}).$$
        \For{$i=1:n$}
        \State Compute the pilot derivatives $\tilde{g}_i=\cl_1'(\tilde{f}(x_i),y_i)$.
        \State Construct Rademacher sequence $\cbr{\varepsilon_i}_{i=1}^{n}$.
        \State Calculate the wild responses $y_i^\dia$ and $y_i^\sh$ in either Scheme I or Scheme II.
        \begin{itemize}
            \item[Scheme I:] $\cl_1'(\hat{f}(x_i),y_i^\dia)=\cl_1'(\hat{f}(x_i),y_i)-2\rho_1\varepsilon_i\tilde{g}_i,\ \cl_1'(\hat{f}(x_i),y_i^\sh)=\cl_1'(\hat{f}(x_i),y_i)+2\rho_2\varepsilon_i\tilde{g}_i,$
            \item[Scheme II:] $\cl_1'(\hat{f}(x_i),y_i^\dia)=-2\rho_1\varepsilon_i\tilde{g}_i,\ \cl_1'(\hat{f}(x_i),y_i^\sh)=2\rho_2\varepsilon_i\tilde{g}_i.$
        \end{itemize}
        \State Append $(x_i, y_i^\diamond)$, $(x_i,y_i^\sh)$ to $\mathcal{D}^\dia$ and $\cD^\sh$:
        $\mathcal{D}^\dia \gets \mathcal{D}^\dia \cup \{(x_i, y_i^\diamond)\},\ \mathcal{D}^\sh \gets \mathcal{D}^\sh \cup \{(x_i, y_i^\sh)\}.$
        \EndFor
        \State Compute the refitted wild solutions $f^{\diamond}_{\rho_1}=\Alg(\cD^\dia)$, $f^{\sh}_{\rho_2}=\Alg(\cD^\sh)$.
        \State Output $\hat{f}$, $f^\diamond_{\rho_1}$, $f^\sh_{\rho_2}$ $\cD^\dia$, $\cD^\sh$, $\cD_0$.
\end{algorithmic}
\end{algorithm}
\begin{remark}
    Frequently, we could take $\tilde{f}=\hat{f}$ as the pilot predictor in empirical applications.
\end{remark}
\iffalse
\paragraph{Carrying Out the Derivative Perturbation}

At the end of this section, we present a practical procedure for implementing the derivative perturbation concretely. Assume that, for each fixed $z$, the loss function $\ell(z,\cdot)$ is strictly convex. Let $\ell^*(z,\cdot)$ denote the Fenchel conjugate of $\ell(z,\cdot)$. By standard results from convex analysis \citep{magaril2003convex}, we have
\[
\frac{d\ell^*(z,u)}{du} = (\nabla \ell(z,\cdot))^{-1}(u).
\]
Therefore, our perturbation schemes can be implemented explicitly as follows:
\[
\textnormal{I:}\qquad
y_i^\diamond
=
\frac{d\ell^*(\hat{f}(x_i),u)}{du}
\Big|_{u=\ell_1'(\hat{f}(x_i),y_i)-2\rho\varepsilon \tilde{g}_i},
\qquad
y_i^\sharp
=
\frac{d\ell^*(\hat{f}(x_i),u)}{du}
\Big|_{u=\ell_1'(\hat{f}(x_i),y_i)+2\rho\varepsilon \tilde{g}_i},
\]
and
\[
\textnormal{II:}\qquad
y_i^\diamond
=
\frac{d\ell^*(\hat{f}(x_i),u)}{du}
\Big|_{u=-2\rho\varepsilon \tilde{g}_i},
\qquad
y_i^\sharp
=
\frac{d\ell^*(\hat{f}(x_i),u)}{du}
\Big|_{u=2\rho\varepsilon \tilde{g}_i}.
\]
Hence, as long as the Fenchel conjugate of the loss function can be computed, the proposed perturbation method can be implemented efficiently in practice.
\fi
\section{Statistical Guarantees}\label{sec:statistical_guarantees}
In this section, we demonstrate that the excess risk can be tightly bounded by the outputs of Algorithm~\ref{alg:wild-refitting}.
\subsection{Evaluating the Excess Risk in Fixed Design}\label{subsec:sta_gua_fixed_design}
We now bound the excess risk in the fixed design. Before we dive into the concrete results, we first provide several important definitions and mathematical notations that are useful in this section. Recall that we need to bound the \emph{true optimism} term $\Opt^*(\hat{f})=\frac{1}{n}\sum_{i=1}^{n}\rbr{\cl_1'(f^*(x_i),y_i)}\rbr{\hat{f}(x_i)-f^*(x_i)}$.

When the model class $\cF$ is mis-specified, as the sample size $n\rightarrow\infty$, the predictor $\hat{f}$ will not converge to $f^*$, instead, it converges to the best approximator within the model class $\cF$, which is denoted by $f^\dagger$, i.e., 
\[
f^\dagger\in\argmin_{f\in\cF}\cbr{\frac{1}{n}\sum_{i=1}^{n}\EE_{y_i}[\cl(f(x_i),y_i)|x_i]}.
\]
We denote the empirical $L_2$ distance between $\hat{f}$ and $f^\dagger$ as $\hat{r}_n$, i.e., $\hat{r}_n=||\hat{f}-f^\dagger||_n$. Correspondingly, we name the following term as \emph{oracle optimism}.
\[
\Opt^\dagger(\hat{f}):=\frac{1}{n}\sum_{i=1}^{n}\cl_1'(f^*(x_i),y_i)(\hat{f}(x_i)-f^\dagger(x_i)).
\]
Moreover, denoting $\cbr{f\in\cF:\|f-g\|_n\le r}$ as $\cB_r(g)$, we define the following empirical processes. Denoting $\cl_1'(\tilde{f}(x_i),y_i)$ as $\tilde{g}_i$, we define $W_n(r)$ and $T_n(r)$ as the following empirical processes:
\[
W_n(r)=\sup_{f\in\cB_r(\hat{f})}\cbr{\frac{1}{n}\sum_{i=1}^{n}\varepsilon_i\tilde{g}_i(f(x_i)-\hat{f}(x_i))},\ T_n(r)=\sup_{f\in\cB_r(\hat{f})}\cbr{\frac{1}{n}\sum_{i=1}^{n}\varepsilon_i\tilde{g}_i(\hat{f}(x_i)-f(x_i))}.
\]
To induce the randomized symmetrization, for any index $i$, we use $w_i$ to denote the random variable $\cl_1'(f^*(x_i),y_i)$. $w_i'$ is an independent copy of $w_i$. With this independent copy, we construct a symmetric random variable $\tilde{w}_i:=\frac{w_i-w_i'}{2}$, and define the following empirical process:
$$Z_n^\varepsilon(r)=\sup_{f\in\cB_r(f^\dagger)}\cbr{\frac{1}{n}\sum_{i=1}^{n}\varepsilon_i\tilde{w}_i(f(x_i)-f^\dagger(x_i))}.$$
Traditionally, bounding empirical processes such as $W_n(r)$ and $T_n(r)$ relies on detailed knowledge of the underlying function class—through measures like the Rademacher complexity, or the VC dimension. In contrast, our doubly wild refitting procedure will show that these quantities can be controlled without explicit access to such structural information. Instead, we establish that the empirical processes can be bounded at certain radii directly using the outputs generated by Algorithm~\ref{alg:wild-refitting} as long as our scaling is appropriate. To formalize this idea, we introduce the following quantities, which are referred to as \emph{wild optimisms}:
\[
\Opt^\dia(f_{\rho_1}^\dia):=\frac{\beta}{4\rho_1n}||f_{\rho_1}^\dia-\hat{f}||_n^2+\frac{1}{2\rho_1 n}\sum_{i=1}^{n}\cl(\hat{f}(x_i),y_i^\dia)-\frac{1}{2\rho_1n}\sum_{i=1}^{n}\cl(f_{\rho_1}^\dia(x_i),y_i^\dia),
\]
\[
\Opt^\dia(f_{\rho_2}^\sh):=\frac{\beta}{4\rho_2n}||f_{\rho_2}^\sh-\hat{f}||_n^2+\frac{1}{2\rho_2 n}\sum_{i=1}^{n}\cl(\hat{f}(x_i),y_i^\sh)-\frac{1}{2\rho_2n}\sum_{i=1}^{n}\cl(f_{\rho_2}^\sh(x_i),y_i^\sh),
\]
Then, the following key Lemma \ref{lemma:bound_W_n-T_n} bounds the empirical processes $W_n$ and $T_n$ by the wild optimism.
\begin{lemma}\label{lemma:bound_W_n-T_n}
 For any noise scale, we have that
    \[
    W_n(\|f_{\rho_1}^\dia-\hat{f}\|_n)\le \Opt^\dia(f_{\rho_1}^\dia),\ T_n(\|f_{\rho_2}^\sh-\hat{f}\|_n)\le \Opt^\sh(f_{\rho_2}^\sh).
    \]
\end{lemma}
The proof is deferred to Appendix \ref{app:proofs_fixed_design}. Lemma \ref{lemma:bound_W_n-T_n} establishes that, at certain radii, the empirical process can be upper bounded by the quantities $\Opt^\dia(f_{\rho_1}^\dia)$ and $\Opt^\sh(f_{\rho_2}^\sh)$, both of which can be computed efficiently from the outputs of Algorithm \ref{alg:wild-refitting}. This lemma serves as a key theoretical component underpinning the wild-refitting procedure, because it enables us to upper bound the empirical processes without knowledge of the underlying function class.

With Lemma \ref{lemma:bound_W_n-T_n}, we have the following theorem regarding the evaluation of the excess risk $\cE_{fix}(\hat{f})$. The proof of this theorem is deferred to Appendix \ref{app:proofs_sec:statistical_guarantee}.
\begin{theorem}\label{thm:fixed_design}
    For any radius $r\ge \hat{r}_n$, let $\rho_1$ and $\rho_2$ be the noise scales for which $\|f_{\rho_1}^\dia-\hat{f}\|$ and $ \|f_{\rho_2}^\sh-\hat{f}\|$ are equal to $2r$. Then, for any $t>0$, with probability at least $1-
    6e^{-t^2}$,
    \begin{align}\label{ineq:main_thm1}
    \cE_{fix}(\hat{f})\le& \frac{\beta}{\alpha}\bar{\cE}_{\cD}(\hat{f})+\frac{\beta}{\alpha}\rbr{\Opt^\dia(f^\dia_{\rho_1})+\Opt^\sh(f^\sh_{\rho_2})+B_n^\dia(\hat{f})+B_n^\sh(\hat{f})}\nonumber \\
    +&[(3\sqrt{\log n}+9)r+\|f^\dagger-f^*\|_n]\frac{2\sqrt{2}\beta\sigma t}{\alpha\sqrt{n}}.
    \end{align}
    In this bound, $\Opt^\dia(f^\dia_{\rho_1})$ and $\Opt^\sh(f^\sh_{\rho_2})$ are wild-optimism terms, and
    $$B_n^\dia(\hat{f}):=\sup_{f\in\cB_{2r}(\hat{f})}\cbr{\frac{1}{n}\sum_{i=1}^{n}\varepsilon_i\rbr{\cl_1'(f^*(x_i),y_i)-\cl_1'(\tilde{f}(x_i),y_i)}(f(x_i)-\hat{f}(x_i))};$$
    \[
    B_n^\sh(\hat{f}):=\sup_{f\in\cB_{2r}(\hat{f})}\cbr{\frac{1}{n}\sum_{i=1}^{n}\varepsilon_i\rbr{\cl_1'(f^*(x_i),y_i)-\cl_1'(\tilde{f}(x_i),y_i)}(\hat{f}(x_i)-f(x_i))}.
    \]
\end{theorem}
Throughout the paper, $B_n^\dia(\hat{f})$ and $B_n^\sh(\hat{f})$ are named as \emph{pilot error terms}.

There are four parts on the right hand side of Inequality \ref{ineq:main_thm1}. The first part is $\frac{\beta}{\alpha}\bar{\cE}_{\cD}(\hat{f})$, which is the empirical excess risk. The second part is the sum of the wild optimism terms $\Opt^\dia(f^\dia_{\rho_1})$ and $\Opt^\sh(f^\sh_{\rho_2})$, which is the most non-trivial component related to bounding the true optimism $\Opt^*(\hat{f})$. The third part is the probabilistic deviation term $[(3\sqrt{\log n}+9)r+\|f^\dagger-f^*\|_n]\frac{2\sqrt{2}\beta\sigma t}{\alpha\sqrt{n}}$, decaying at a parametric rate \citep{rakhlin2022mathstat}. 

Lastly, we have the pilot error terms $B_n^\dia(\hat{f})$ and $ B_n^\sh(\hat{f})$ as the fourth part. Bounding the pilot error terms can be tricky. In \citet{wainwright2025wild}, only an intuitive explanation was provided, suggesting that these terms might be upper bounded by wild optimism. Our next theorem rigorously establishes that, under suitable conditions, these terms are dominated by the wild optimism in Theorem \ref{thm:fixed_design}, up to a multiplicative factor.
\begin{theorem}\label{thm:bounding_pilot_error}
    Let $\cX\subset \RR^m$ be compact, and let $x_1,\cdots,x_n\in\cX$ be pairwise distinct covariates. Suppose the pre-trained recentering pilot predictor $\tilde{f}$ satisfies $\|\tilde{f}-f^*\|_n\le M$ for some $M>0$. Assuming $|\cl_1'(\cdot,\cdot)|$ is bounded by $L$, with the conditional variance $\Var(\cl_1'(\tilde{f}(x),y)|x)$ lower bounded by $\tau^2$, i.e., $\Var(\cl_1'(\tilde{f}(x),y)|x)\ge \tau^2>0$. Then, $\forall \delta\in(0,1)$, when $n\ge \frac{L\log(1/\delta)}{\tau^4}$, with probability at least $1-2\delta$,
    \[
    B_n^\dia(\hat{f})+B_n^\sh(\hat{f})\lesssim \cO\rbr{\frac{4\beta M}{\tau}\rbr{\Opt^\dia(f^\dia_{\rho_1})+\Opt^\sh(f^\sh_{\rho_2})}}.
    \]
\end{theorem}
Theorem \ref{thm:bounding_pilot_error} shows that when the training architecture $\cF$ is sufficiently rich and the sample size is large enough, the pilot error terms are of no greater order than the wild optimism. To the best of our knowledge, Theorem~\ref{thm:bounding_pilot_error} is the first result to establish a theoretically rigorous bound for these terms. This contribution fills an important gap in the existing wild-refitting theory. A detailed proof is given in Appendix~\ref{app:proofs_sec:statistical_guarantee}.

\subsection{Bounding the Distance between $\hat{f}$ and $f^\dagger$}\label{subsec:bounding_hatr}
Theorem \ref{thm:fixed_design} is useful only when $\hat{r}_n=\|\hat{f}-f^\dagger\|_n$ is known. In this subsection, we provide an effective method for bounding this quantity. For the simplicity of discussion, we consider a well-specified scenario in this subsection, i.e., $f^*\in\cF$ and therefore $f^\dagger=f^*$. Then, we have the following theorem about bounding $\hat{r}_n$.

\begin{theorem}\label{thm:bounding_hatr}
    For any $t>0$, with probability at least $1-2e^{-t^2}$, we have that
    \begin{align*}
\hat{r}_n\le& \sqrt{\frac{2}{\alpha}\rbr{W_n(2\hat{r}_n))+T_n(2\hat{r}_n)+B_n^\dia(\hat{f})+B_n^\sh(\hat{f})}}+\rbr{\frac{6\sqrt{\log n}+14}{\alpha}+1}\frac{2\sqrt{2}\sigma t}{\sqrt{n}}\\
\le&\sqrt{\frac{2}{\alpha}\rbr{W_n(2\hat{r}_n))+T_n(2\hat{r}_n))}}+\sqrt{\frac{2}{\alpha}(B_n^\dia(\hat{f})+B_n^\sh(\hat{f}))}+\rbr{\frac{6\sqrt{\log n}+14}{\alpha}+1}\frac{2\sqrt{2}\sigma t}{\sqrt{n}}.
\end{align*}
\end{theorem}
We emphasize that, in contrast to \citet{wainwright2025wild}, which assumes that the difference between the predictor trained on a noisy dataset and that trained on a noiseless dataset is upper bounded by the noise scale, Theorem~\ref{thm:bounding_hatr} removes this assumption while still providing a valid bound on $\hat{r}_n$ under the well-specified setting.

To illustrate how to apply Theorem~\ref{thm:bounding_hatr} to bound $\hat{r}_n$, observe that the right-hand side of the theorem contains three components. The first part is $\sqrt{W_n(2\hat{r}_n)+T_n(2\hat{r}_n)}$, which is closely connected with the \emph{wild optimism} term. The second part is the pilot error $B_n^\dia(\hat{f})+B_n^\sh(\hat{f})$. The last part is the probability deviation term $\rbr{\frac{6\sqrt{\log n}+14}{\alpha}+1}\frac{2\sqrt{2}\sigma t}{\sqrt{n}}$. Among these, whenever the function class $\mathcal{F}$ is more complex than a simple parametric family, the dominant term will be $\sqrt{\,W_n(2\hat{r}_n)+T_n(2\hat{r}_n)\,}$.

Ignoring the lower order terms, we have the following approach: tuning the radius parameter $r$ such that
\begin{align}\label{ineq:hatr1}
r \;\le\; \sqrt{\frac{2}{\alpha}\Bigl(W_n(2r)+T_n(2r)\Bigr)}.
\end{align}
Any such $r$ then provides a valid upper bound for $\hat{r}_n$. Importantly, this procedure is fully data-driven, since given any $s$, by definition, both $W_n(s)$ and $T_n(s)$ are computable solely from the trained predictor $\hat{f}$ and the loss derivatives of the pilot predictor $\{\tilde{g}_i\}_{i=1}^n$.

In fact, we also have the following corollary with a more interpretable bound. 

\begin{corollary}\label{cor:bounding_hatr}
    For any noise scale $\rho_1$, $\rho_2$, for any $t>0$, with probability at least $1-2e^{-t^2}$,
    \begin{align*}
    \hat{r}_n^2\le &\max\cbr{(r_{\rho_1}^\dia)^2, (r_{\rho_2}^\sh)^2,\rbr{\frac{4W_n(2r_{\rho_1}^\dia)}{\alpha r_{\rho_1}^\dia}+\frac{4T_n(2r_{\rho_2}^\sh)}{\alpha r_{\rho_2}^\sh}}\cdot\hat{r}_n}\\
    +&\frac{4(B_n^\dia(\hat{f})+B_n^\sh(\hat{f}))}{\alpha}+\rbr{\frac{6\sqrt{\log n}+14}{\alpha}+1}^2\frac{16\sigma^2 t^2}{n}.
    \end{align*}
\end{corollary}
The bound in Corollary \ref{cor:bounding_hatr} is slightly weaker compared to Theorem \ref{thm:bounding_hatr} but more interpretable because we no longer need to tune the radius $r$ to solve Inequality \ref{ineq:hatr1}. Instead, we just need to tune the noise scale. If we disregard the probability deviation and pilot error terms, we can roughly summarize Corollary \ref{cor:bounding_hatr} in the following inequality:
\[
\hat{r}_n\le \max\cbr{r_{\rho_1}^\dia, r_{\rho_2}^\sh, \frac{4W_n(2r_{\rho_1}^\dia)}{\alpha r_{\rho_1}^\dia}+\frac{4T_n(2r_{\rho_2}^\sh)}{\alpha r_{\rho_2}^\sh}}.
\]
This inequality allows us to tune $\rho_1$ and $\rho_2$ to find a valid and tight upper bound of $\hat{r}_n$ more easily.

Putting Theorem \ref{thm:fixed_design}, Theorem \ref{thm:bounding_pilot_error} and Theorem \ref{thm:bounding_hatr} together, for any $t>0$, with probability at least $1-10e^{-t^2}$, when $n\gtrsim ct^2$, under some regularization conditions, we have
\[
\cE_{fix}(\hat{f})\lesssim \cO\rbr{\Bar{\cE}_{\cD}(\hat{f})+\Opt^\dia(f^\dia_{\rho_1})+\Opt^\sh(f^\sh_{\rho_2})},
\] 
for any $\rho_1,\rho_2$ such that $\|f^\dia_{\rho_1}-\hat{f}\|_n=\|f^\dia_{\rho_1}-\hat{f}\|_n\lesssim \max \cbr{r_{\rho_1}^\dia, r_{\rho_2}^\sh, \frac{4W_n(2r_{\rho_1}^\dia)}{\alpha r_{\rho_1}^\dia}+\frac{4T_n(2r_{\rho_2}^\sh)}{\alpha r_{\rho_2}^\sh}}$.

Therefore, we finally achieve an efficient refitting procedure for bounding the excess risk of any machine learning algorithm with rigorous statistical guarantees under the fixed design setting. Treating the procedure as a black-box,our method fully utilizes the single dataset and does not depend on the concrete training architecture.
\section{Discussion}
In this paper, we investigate a doubly wild refitting procedure for upper bounding the excess risk of arbitrary black-box ERM algorithms in machine learning. We show that the core mechanism underlying wild refitting is the perturbation of derivatives, through which our framework subsumes \citet{wainwright2025wild} as special cases.

Unlike traditional approaches that rely on explicit complexity measures of the underlying hypothesis class, our method bypasses such requirements entirely. Instead, we leverage the Rademacher symmetrization of the derivative vectors and control the resulting empirical processes using perturbed pseudo-outcomes, the two refitted wild predictors, and the original trained predictor.

Several promising directions for future research remain open. First, our analysis currently focuses on the fixed-design setting; extending wild refitting to the random-design regime would be interesting. Second, developing wild refitting methods for ERM with penalization is also an important and technically challenging question. Moreover, existing methods of wild refitting require full data memorization for derivative perturbation at every data point, and designing a data-efficient variant would be very appealing. More broadly, it is natural to ask whether wild refitting can be extended to other settings, such as high-dimensional prediction or set-valued estimators. Finally, empirical studies on real-world trained AI models would be valuable for assessing the practical performance of wild refitting and demonstrating its effectiveness. We leave these directions for future work.
\clearpage

\bibliographystyle{plainnat}
\bibliography{haichen/sections/refs}
\newpage
\appendix
\section{Useful Mathematical Tools}\label{app:useful_math_tools}
In this section, we introduce some background knowledge about variational analysis, operator theory, probability, and interpolation theory.
\begin{definition}[Gâteaux Derivative]
    Let $X$ be a real Banach space, $U\subset X$ open, and $F:U\to\RR$.
We say $F$ is Gâteaux differentiable at $x\in U$ if for every $h\in X$ the
directional limit
\[
F'(x)[h]:=\lim_{t\downarrow 0}\frac{F(x+th)-F(x)}{t}
\]
exists and the map $h\mapsto F'(x)[h]$ is \emph{positively homogeneous} and \emph{additive on rays}.
If, in addition, $h\mapsto F'(x)[h]$ is continuous and linear, we identify it with an element
$F'_G(x)\in X^\ast$ (the Gâteaux derivative) via $F'(x)[h]=\langle F'_G(x),h\rangle$.
\end{definition}
\begin{definition}[Monotone and strictly monotone mapping]
Let $A:\RR^m\to\RR^m$ be a mapping.
\begin{itemize}
  \item $A$ is \emph{monotone} if $\langle A(x)-A(y),\,x-y\rangle\ge 0,\ 
  \forall\,x,y\in\RR^m.$
  \item $A$ is \emph{strictly monotone}  if
  $\langle A(x)-A(y),\,x-y\rangle >0,\ \forall\,x\neq y\in\RR^m.$
\end{itemize}
\end{definition}
\begin{definition}[Resolvent]
For $\lambda>0$, the \emph{resolvent} of $A$ (when it exists globally) is the mapping
\[
J_\lambda^A := (I+\lambda A)^{-1}:\RR^m\to\RR^m,\ \text{i.e.}\ x = J_\lambda^A(b)\ \text{iff}\ x+\lambda A(x)=b.
\]
\end{definition}
\begin{lemma}\label{lemma:sub_gaussian_concentration}
    Let $X_1,\dots,X_n \in \mathbb{R}^d$ be independent random vectors satisfying 
$\EE[X_i]=0,$ and for all unit vectors $u\in\mathbb{S}^{d-1}$, $\langle u,X_i\rangle$ is $\sigma^2$sub-Gaussian. Then, for any $\delta\in(0,1)$, with probability at least $1-\delta$, we have
\[
\max_{1\le i\le n}\|X_i\|_2
\;\le\;
\sigma\,\sqrt{\,8\big(d\log 5+\log\tfrac{2n}{\delta}\big)}.
\]
\end{lemma}
\begin{lemma}\label{lemma:gateaux_first_order}
Let $X$ be a real Banach space, $U\subset X$ open, and $F:U\to\RR$ be Gâteaux differentiable at $x^\ast\in U$.
If $x^\ast$ is a local minimizer of $F$ on $U$, then
\[
F'(x^\ast)[h]\ge 0\ \text{and}\ F'(x^\ast)[-h]\ge 0
\qquad(\forall\,h\in X),
\]
hence $F'(x^\ast)[h]=0\quad(\forall\,h\in X)$.
In particular, if $X$ is a Hilbert space, when $F'(x^\ast)[\cdot]$ is represented by a continuous linear functional via the Riesz representation theorem \citep{riesz1909genesis},
$F'_G(x^\ast)\in X^\ast$, the condition is equivalent to $F'_G(x^\ast)=\mathbf{0}\in X^\ast$.
\end{lemma}
\lemma[Browder-Minty theorem]\label{lemma:Browder-Minty}
    A bounded, continuous, coercive and monotone function $T$ from a real, separable reflexive Banach space $X$ into its continuous dual space $X^*$ is automatically surjective, i.e., for any $g\in X^*,\ \exists u\in X$ s.t. $T(u)=g$. Moreover, if $T$ is strictly monotone, then $T$ is also injective and thus bijective.
\endlemma
The Browder-Minty theorem can be found in Theorem 9.14 in \citet{ciarlet2025linear}.
\begin{lemma}\label{lemma:lipschitz_concentration}
   Let $X_1,\ldots,X_n \in \mathbb{R}^d$ be independent, mean-zero random vectors satisfying the
$\sigma^2$-sub-Gaussian moment generating function bound:
\[
\mathbb{E}\exp\!\big(t\,\langle u,X_i\rangle\big)
\;\le\; \exp\!\left(\tfrac{\sigma^2 t^2}{2}\right)
\quad \text{for all } t\in\mathbb{R},\ u\in\mathbb{S}^{d-1},\ i=1,\ldots,n.
\]
Let $f:(\mathbb{R}^d)^n \to \mathbb{R}$ be $L$-Lipschitz with respect to the product Euclidean norm:
\[
|f(x_1,\ldots,x_n)-f(y_1,\ldots,y_n)|
\;\le\;
L\Big(\sum_{i=1}^n \|x_i-y_i\|_2^2\Big)^{1/2}.
\]

Then, for all $\lambda\in\mathbb{R}$, we have $\mathbb{E}\exp\!\Big(\lambda\big(f(X_1,\ldots,X_n)-\mathbb{E}f(X_1,\ldots,X_n)\big)\Big)
\;\le\;
\exp\!\left(\frac{\sigma^2 L^2 \lambda^2}{2}\right).$

Consequently, for all $t\ge 0$,
\[
\mathbb{P}\!\left(\big|f(X_1,\ldots,X_n)-\mathbb{E}f(X_1,\ldots,X_n)\big|\ge t\right)
\;\le\;
2\exp\!\left(-\frac{t^2}{2\sigma^2 L^2}\right).
\]
\end{lemma}
Lemma \ref{lemma:lipschitz_concentration} is a concentration inequality using Lipschitz continuity. See \citet{vershynin2018high} for reference.
\lemma[Hanson-Wright Inequality]\label{lemma:hanson-wright}
Let $w=(w_1,\cdots,w_n)$ be a random vector with independent, zero-mean, $\sigma^2$ sub-Gaussian entries. Then, for any matrix $A\in\RR^{n\times n}$ and any $t>0$,
\[
\PP(|w^TAw-\EE[w^TAw]|\ge t)\le 2\exp\rbr{-c\min\cbr{\frac{t^2}{\sigma^4\|A\|_F^2},\frac{t}{\sigma^2\|A\|_{op}}}}.
\]
\endlemma
Lemma \ref{lemma:hanson-wright} can be found in \citet{rudelson2013hanson} and we omit the proof here.
\begin{definition}\label{def:univ_approx}
    We say that a function family $\cG$ is an universal approximation class on a compact set $\cX\subset \RR^d$ if $\forall \epsilon>0$, $\forall h\in C(\cX)$, $\exists g\in\cG$ such that $\sup_{x\in\cX}|g(x)-h(x)|\le\epsilon$.
\end{definition}
\begin{lemma}[Shepard's Interpolation theorem]\label{thm:shepard_interpolation}
    Let $\{x_1, x_2, \dots, x_N\} \subset \mathbb{R}^d$ be a set of $N$ distinct data points, and let $\{y_1, y_2, \dots, y_N\} \subset \mathbb{R}$ be the corresponding scalar values. Set $w_i(x) = \frac{1}{\|x - x_i\|^2}$ and $\psi_i(x) = \frac{w_i(x)}{\sum_{j=1}^N w_j(x)}$ for $i=1,\cdots,n$. Define the function $F: \mathbb{R}^d \to \mathbb{R}$ as:
\begin{equation*}
    F(x) = \begin{cases} 
        \displaystyle \sum_{i=1}^N \psi_i(x) y_i & \text{if } x \notin X \\
        y_k & \text{if } x = x_k \text{ for some } k \in \{1, \dots, N\}
    \end{cases}
\end{equation*}
Then $F(x)$ is differentiable on $\mathbb{R}^d$ and satisfies the interpolation condition $F(x_i) = y_i$ for all $i$.
\end{lemma}
The proof of Theorem \ref{thm:shepard_interpolation} uses the partition of unity technique can be found in \citet{shepard1968two}.
\section{Proofs in Section \ref{sec:model}}\label{app:proofs_sec:model}
\begin{proof}[Proof of Proposition \ref{prop:gateaux1}]
    We notice that $F(f):=\EE_{y}[\cl(f(x),y)|x]$ is a functional of $f$. We take the Gâteaux derivative of this functional and apply Lemma \ref{lemma:gateaux_first_order} at the optimal predictor $f^*$ to get that
    \[
    F'(f^*)[h]=\EE_{y}[\cl_1'(f^*(x),y)h(x)|x]=0,\ \forall h.
    \]
    Since this holds for all $h$, we have $\EE_{y}[\cl_1'(f^*(x),y)]=0$, and finish the proof.
   \end{proof}
\begin{proof}[Proof of Proposition \ref{prop:gateaux2}]
    Similar to Proposition \ref{prop:gateaux1}, we define the functional 
    \[
    H(f):=\frac{1}{n}\sum_{i=1}^{n}\cl(f(x_i),y_i).
    \]
    By the condition $\cF$ is convex and Banach, we compute the Gâteaux derivative of this functional and apply Lemma \ref{lemma:gateaux_first_order} at the minimizer $\hat{f}$ to get
    \[
    \frac{1}{n}\sum_{i=1}^{n}\cl_1'(\hat{f}(x_i),y_i)(g(x_i)-\hat{f}(x_i))=0,\ \forall g\in\cF. 
    \]
    We finish the proof.
   \end{proof}
\begin{proof}[Proof of proposition \ref{prop:bound_excess_by_empirical}]
The proof is simply algebra, which is presented as below. By the Bregman representation of the loss function, we have
    \begin{align*}
    \cE_{fix}(\hat{f})&=\frac{1}{n}\sum_{i=1}^{n}\EE_{y_i}[\cl_1'(f^*(x_i),y_i)(\hat{f}(x_i)-f^*(x_i))+D_{\cl(\cdot,y_i)}(\hat{f}(x_i),f^*(x_i))]\\
    &=\frac{1}{n}\sum_{i=1}^{n}\EE_{y_i}\sbr{\cl_1'(f^*(x_i),y_i)(\hat{f}(x_i)-f^*(x_i))+D_{\cl(\cdot,y_i)}(\hat{f}(x_i),f^*(x_i))}\\
    &=\frac{1}{n}\sum_{i=1}^{n}\EE_{y_i}[D_{\cl(\cdot,y_i)}(\hat{f}(x_i),f^*(x_i))]\\
    &\le \frac{1}{n}\sum_{i=1}^{n}\frac{\beta}{2}(\hat{f}(x_i)-f^*(x_i))^2\\
    &\le \frac{1}{n}\sum_{i=1}^{n}\frac{\beta}{\alpha}\sbr{\cl(\hat{f}(x_i),y_i)-\cl(f^*(x_i),y_i)+\cl_1'(f^*(x_i),y_i)(\hat{f}(x_i)-f^*(x_i))}\\
    &=\frac{\beta}{\alpha}\bar{\cE}_{\cD}(\hat{f})+\frac{\beta}{\alpha}\frac{1}{n}\sum_{i=1}^{n}\cl_1'(f^*(x_i),y_i)(\hat{f}(x_i)-f^*(x_i)).
\end{align*}
The third equation is given by Proposition \ref{prop:gateaux1}, and the inequalities hold due to the $\beta$-smoothness and $\alpha$-strong convexity of the loss function $\cl$. 
   \end{proof}
\section{Proofs in Section \ref{sec:doubly_wild_refitting}}
\begin{proof}[Proof of Proposition \ref{prop:wild_refit_well_defined}]
For any $u$ fixed, we view $-\cl_1'(u,\cdot)$ as an operator with respect to the second variable. We consider the wild refitting procedure that $$-\cl_1'(\hat{f}(x_i),y_i^\dia)=-(\cl_1'(\hat{f}(x_i),y_i)-2\rho_i\varepsilon_i\tilde{g}_i),$$
$$-\cl_1'(\hat{f}(x_i),y_i^\sh)=-(\cl_1'(\hat{f}(x_i),y_i)+2\rho_i\varepsilon_i\tilde{g}_i).$$
Note that $\RR$ is finite dimensional, and therefore naturally Banach, separable, and reflexive. Moreover, the dual space of $\RR$ is $\RR$ itself. By the condition in Proposition \ref{prop:wild_refit_well_defined}, $-\cl_1'(\hat{f}(x_i),y)$ is continuous and coercive. Then, we can apply Lemma \ref{lemma:Browder-Minty} to conclude that $-\cl_1'(u,\cdot)$ is surjective. Therefore, $\cl_1'(u,\cdot)$ is also surjective. Hence, $y_i^\dia$ and $y_i^\sh$ exist and our wild refitting procedure is well-defined.
   \end{proof}
\section{Proofs in Section \ref{sec:statistical_guarantees}}\label{app:proofs_sec:statistical_guarantee}
\subsection{Proofs in Subsection \ref{subsec:sta_gua_fixed_design}}\label{app:proofs_fixed_design}
\begin{proof}[Proof of Lemma \ref{lemma:bound_W_n-T_n}]
    By Proposition \ref{prop:gateaux2}, we have the first order optimality of $\hat{f}$:
\[
\frac{1}{n}\sum_{i=1}^{n}\cl_1'(\hat{f}(x_i),y_i)(f(x_i)-\hat{f}(x_i))=0,\ \forall\ f\in\cF.
\]

Then, by the $\beta$-smoothness of $\cl$, we have that
\begin{align}\label{ineq:lemma1pf_1}
    &\frac{1}{n}\sum_{i=1}^{n}\cl(f(x_i),y_i^\dia)\nonumber\\
    \le& \frac{1}{n}\sum_{i=1}^{n}\cl(\hat{f}(x_i),y_i^\dia)+ \frac{2\rho_1}{n}\sum_{i=1}^{n}\cl_1'(\hat{f}(x_i),y_i^\dia)(f(x_i)-\hat{f}(x_i))+\frac{\beta}{2}\frac{1}{n}\sum_{i=1}^{n}(f(x_i)-\hat{f}(x_i))^2\nonumber\\
    =&\frac{1}{n}\sum_{i=1}^{n}\cl(\hat{f}(x_i),y_i^\dia)-\frac{2\rho_1}{n}\sum_{i=1}^{n}\varepsilon_i\tilde{g}_i\rbr{f(x_i)-\hat{f}(x_i)}+\frac{\beta}{2}\frac{1}{n}\sum_{i=1}^{n}(f(x_i)-\hat{f}(x_i))^2.
\end{align}
The equality holds under both perturbation schemes I) and II).
Similarly, by the construction of $y^\sh_i$, we have that 
\begin{align}\label{ineq:lemma1pf_2}
    \frac{1}{n}\sum_{i=1}^{n}\cl(f(x_i),y_i^\sharp)\le& \frac{1}{n}\sum_{i=1}^{n}\cl(\hat{f}(x_i),y_i^\sh)-\frac{2\rho_1}{n}\sum_{i=1}^{n}\varepsilon_i\tilde{g}_i\rbr{\hat{f}(x_i)-f(x_i)}\nonumber\\
    +&\frac{\beta}{2}\frac{1}{n}\sum_{i=1}^{n}(f(x_i)-\hat{f}(x_i))^2.
\end{align}
Rearranging inequalities \ref{ineq:lemma1pf_1} and\ref{ineq:lemma1pf_2}, we have that
\begin{align}\label{ineq:key_lemma_1}
     &\frac{1}{n}\sum_{i=1}^{n}\cl(f(x_i),y_i^\dia)- \frac{1}{n}\sum_{i=1}^{n}\cl(\hat{f}(x_i),y_i^\dia)\nonumber\\
     \le& \frac{2\rho_1}{n}\sum_{i=1}^{n}\varepsilon_i\tilde{g}_i(f(x_i)-\hat{f}(x_i))+\frac{\beta}{2}\frac{1}{n}\sum_{i=1}^{n}(f(x_i)-\hat{f}(x_i))^2,
\end{align}
\begin{align}\label{ineq:key_lemma_2}
    &\frac{1}{n}\sum_{i=1}^{n}\cl(f(x_i),y_i^\sharp)- \frac{1}{n}\sum_{i=1}^{n}\cl(\hat{f}(x_i),y_i^\sh)\nonumber\\
    \le& \frac{2\rho_2}{n}\sum_{i=1}^{n}\varepsilon_i\tilde{g}_i(\hat{f}(x_i)-f(x_i))+\frac{\beta}{2}\frac{1}{n}\sum_{i=1}^{n}(f(x_i)-\hat{f}(x_i))^2.
\end{align}
Notice that 
\[
\argmin_{f\in\cF}\cbr{\frac{1}{n}\sum_{i=1}^{n}\cl(f(x_i),y_i^\dia)}=\argmin_{f\in\cF}\cbr{\frac{1}{n}\sum_{i=1}^{n}\cl(f(x_i),y_i^\dia)-\frac{1}{n}\sum_{i=1}^{n}\cl(\hat{f}(x_i),y_i^\dia)},
\]
and
\[
\argmin_{f\in\cF}\cbr{\frac{1}{n}\sum_{i=1}^{n}\cl(f(x_i),y_i^\sh)}=\argmin_{f\in\cF}\cbr{\frac{1}{n}\sum_{i=1}^{n}\cl(f(x_i),y_i^\sh)-\frac{1}{n}\sum_{i=1}^{n}\cl(\hat{f}(x_i),y_i^\sh)}.
\]
Taking the minimization on both sides about inequality \ref{ineq:key_lemma_1} and inequality \ref{ineq:key_lemma_2}, and notice that $f^\dia_{\rho_1}$ and $f^\sh_{\rho_2}$ are corresponding minimizers, we have that
\[
\frac{1}{n}\sum_{i=1}^{n}\cl(f_{\rho_1}^\dia(x_i),y_i^\dia)-\frac{1}{n}\sum_{i=1}^{n}\cl(\hat{f}(x_i),y_i^\dia)\le \frac{\beta}{2n}||f_{\rho_1}^\dia-\hat{f}||_n^2-2\rho_1 W_n(||f_{\rho_1}^\dia-\hat{f}||_n).
\]
\[
\frac{1}{n}\sum_{i=1}^{n}\cl(f_{\rho_2}^\sh(x_i),y_i^\sh)-\frac{1}{n}\sum_{i=1}^{n}\cl(\hat{f}(x_i),y_i^\sh)\le \frac{\beta}{2n}||f_{\rho_2}^\sh-\hat{f}||_n^2-2\rho_2 T_n(||f_{\rho_2}^\sh-\hat{f}||_n),
\]
where $W_n$ and $T_n$ are empirical processes:
\[
W_n(r)=\sup_{f\in\cB_r(\hat{f})}\cbr{\frac{1}{n}\sum_{i=1}^{n}\varepsilon_i\tilde{g}_i(f(x_i)-\hat{f}(x_i))},
\]
\[
T_n(r)=\sup_{f\in\cB_r(\hat{f})}\cbr{\frac{1}{n}\sum_{i=1}^{n}\varepsilon_i\tilde{g}_i(\hat{f}(x_i)-f(x_i))}.
\]
Therefore, we conclude that 
\[
W_n(\|f^\dia_{\rho_1}-\hat{f}\|_n)\le \Opt^\dia(f^\dia_{\rho_1}),\ T_n(\|f^\sh_{\rho_2}-\hat{f}\|_n)\le \Opt^\sh(f^\sh_{\rho_2}),
\]
and finish the proof.

\end{proof}
\begin{proof}[Proof of Theorem \ref{thm:fixed_design}]
    
With Lemma \ref{lemma:bound_W_n-T_n}, we now provide the proof of Theorem~\ref{thm:fixed_design}. The argument proceeds by establishing a sequence of intermediate lemmas, which are then combined to derive the main result. The proofs of these lemmas are referred to Appendix \ref{app:proofs_lemmas_in_proofs_sec:statistical_guarantee}.
First, we have the following lemma about connecting the true optimism with the \emph{oracle optimism} $\Opt^\dagger(\hat{f})$.
\begin{lemma}\label{lemma:Opt*<Opt_dagger}
    For any $t>0$, we have that with probability at least $1-e^{-t^2}$, 
    \[
    \Opt^*(\hat{f})\le \Opt^\dagger(\hat{f})+\frac{\sqrt{2}\sigma\|f^\dagger-f^*\|_nt}{\sqrt{n}}.
    \]
\end{lemma}
Lemma \ref{lemma:Opt*<Opt_dagger} enables us to just focus on the term $\Opt^\dagger(\hat{f})$. Specifically, when the radius $r$ is moderately large, we have the following lemma. 
\begin{lemma}\label{lemma:Opt_dagger<Zn}
    For any $r\ge \hat{r}_n=\|\hat{f}-f^\dagger\|_n$, with probability at least $1-e^{-t^2}$,
    \[
    \max\cbr{\Opt^\dagger(\hat{f}),2Z_n^\varepsilon(r)}\le 2\EE_{\varepsilon_{1:n},\tilde{w}_{1:n}}[Z_n^\varepsilon(r)]+\frac{2\sqrt{2}\Bar{\sigma} rt}{\sqrt{n}}.
    \]
\end{lemma}
With Lemma \ref{lemma:Opt_dagger<Zn}, we only need to bound the term $\EE_{\varepsilon_{1:n},\tilde{w}_{1:n}}[Z_n^\varepsilon(r)]$. To handle this, we involve the following intermediate empirical process. We define
\[
U_n^\varepsilon(r):=\sup_{f\in\cB_r(\hat{f})}\cbr{\frac{1}{n}\sum_{i=1}^{n}\varepsilon_i\tilde{w}_i(f(x_i)-\hat{f}(x_i))}.
\]
Then, we can bound $Z_n$ by $U_n$ through the following lemma.
\begin{lemma}\label{lemma:Zn<Un}
we have the deterministic bound that for any $r>0$,
\[
\EE_{\varepsilon_{1:n}}[Z_n^\varepsilon(r)]\le \EE_{\varepsilon_{1:n}}[U_n^\varepsilon(r+\hat{r}_n)],\ \EE_{\varepsilon_{1:n},\tilde{w}_{1:n}}[Z_n^\varepsilon(r)]\le \EE_{\varepsilon_{1:n},\tilde{w}_{1:n}}[U_n^\varepsilon(r+\hat{r}_n)]
\]
Therefore, when $r\ge\hat{r}_n$, we also have 
\[
\EE_{\varepsilon_{1:n},\tilde{w}_{1:n}}[Z_n^\varepsilon(r)]\le \EE_{\varepsilon_{1:n},\tilde{w}_{1:n}}[U_n^\varepsilon(r+\hat{r}_n)]\le \EE_{\varepsilon_{1:n},\tilde{w}_{1:n}}[U_n^\varepsilon(2r)].
\]
\end{lemma}
For $U_n^\varepsilon(r)$, we further have the following concentration lemma.
\begin{lemma}\label{lemma:bound_E[Un]}
    For any $t>0$ and any $r\ge\hat{r}_n$, with probability at least $1-2e^{-t^2}$, 
\begin{align}\label{ineq:E[U]_concentration}
\EE_{\varepsilon_{1:n},\tilde{w}_{1:n}}[U_n^\varepsilon(r)]\le U_n^\varepsilon(r)+\frac{\sqrt{2}\sigma(3\sqrt{\log n}+6)rt}{\sqrt{n}}.
\end{align}
\end{lemma}
Combining Lemma \ref{lemma:Opt*<Opt_dagger}, Lemma \ref{lemma:Opt_dagger<Zn}, Lemma \ref{lemma:Zn<Un}, and Lemma \ref{lemma:bound_E[Un]} together, we arrive at the conclusion that for any $t>0$, with probability at least $1-4e^{-t^2}$, for any $r\ge \hat{r}_n$,
\begin{align}\label{ineq:Opt^*<Un}
\Opt^*(\hat{f})\le& 2U_n^\varepsilon(r+\hat{r}_n)+[(3\sqrt{\log n}+7)r+\|f^\dagger-f^*\|_n]\frac{2\sqrt{2}\sigma t}{\sqrt{n}}\nonumber\\
\le&2U_n^\varepsilon(2r)+[(3\sqrt{\log n}+7)r+\|f^\dagger-f^*\|_n]\frac{2\sqrt{2}\sigma t}{\sqrt{n}}.
\end{align}
Therefore, we only need to bound $U_n^\varepsilon(r)$ by the outputs from Algorithm \ref{alg:wild-refitting}. 
\begin{lemma}\label{lemma:Un<wild_opt}
For any radius $r>0$, let $\rho_1$ and $\rho_2$ be the noise scales such that $$r^\dia_{\rho_1}:=\|f_{\rho_1}^\dia-\hat{f}\|_n=2r,\ r^\sh_{\rho_2}:=\|f_{\rho_2}^\sh-\hat{f}\|_n=2r.$$
Then, for any $t>0$, with probability at least $1-2e^{-t^2}$,
\[
U_n^\varepsilon(2r)\le \frac{1}{2}\rbr{\Opt^\dia(f^\dia_{\rho_1})+\Opt^\sh(f^\sh_{\rho_2})+B_n^\dia(\hat{f})+B_n^\sh(\hat{f})}+\frac{2\sqrt{2}\sigma rt}{\sqrt{n}},
\]
where $B_n^\dia(\hat{f})$ and $B_n^\sh(\hat{f})$ are pilot error terms defined as:
\[B_n^\dia(\hat{f}):=\sup_{f\in\cB_{2r}(\hat{f})}\cbr{\frac{1}{n}\sum_{i=1}^{n}\varepsilon_i\rbr{\cl_1'(f^*(x_i),y_i)-\cl_1'(\hat{f}(x_i),y_i)}(f(x_i)-\hat{f}(x_i))};
    \]
    \[
    B_n^\sh(\hat{f}):=\sup_{f\in\cB_{2r}(\hat{f})}\cbr{\frac{1}{n}\sum_{i=1}^{n}\varepsilon_i\rbr{\cl_1'(f^*(x_i),y_i)-\cl_1'(\hat{f}(x_i),y_i)}(\hat{f}(x_i)-f(x_i))}.
    \]
    as defined in Theorem \ref{thm:fixed_design}.
\end{lemma}
Therefore, combining Inequality \ref{ineq:Opt^*<Un} with Lemma \ref{lemma:Un<wild_opt}, with probability at least $1-6e^{-t^2}$, we have:
\begin{align*}
    \Opt^*(\hat{f})\le \Opt^\dia(f^\dia_{\rho_1})+\Opt^\sh(f^\sh_{\rho_2})+B_n^\dia(\hat{f})+B_n^\sh(\hat{f})+[(3\sqrt{\log n}+9)r+\|f^\dagger-f^*\|_n]\frac{2\sqrt{2}\sigma t}{\sqrt{n}}.
\end{align*}
Combining this result with Proposition \ref{prop:bound_excess_by_empirical}, we complete the proof.

   \end{proof}
\begin{proof}[Proof of Theorem \ref{thm:bounding_pilot_error}]
By the property of the dual norm, we have $\sup_{v:\|v\|_n\le R}\frac{1}{n}\sum_{i=1}^{n}v_ib_i=R\|b\|_n$. Denoting $w=(\cl_1'(f^*(x_1,y_1)),\cdots,\cl_1'(f^*(x_n),y_n));\ \eta=(\cl_1'(\tilde{f}(x_1),y_1),\cdots,\cl_1'(\tilde{f}(x_n),y_n)),$
we have
\begin{align*}
\sup_{f\in\cB_{2r}(\hat{f})}\cbr{\frac{1}{n}\sum_{i=1}^{n}\varepsilon_i\rbr{\cl_1'(f^*(x_i),y_i)-\cl_1'(\tilde{f}(x_i),y_i)}(\hat{f}(x_i)-f(x_i))}\le 2r\|\eta-w\|_n.
\end{align*}
By the definition of empirical processes $W_n$ and $T_n$, we know that, 
\[
W_n(r)=\sup_{f\in\cB_{r}(\hat{f})}\cbr{\frac{1}{n}\sum_{i=1}^{n}\varepsilon_i\tilde{g}_i(f(x_i)-\hat{f}(x_i))},\ T_n(r)=\sup_{f\in\cB_{r}(\hat{f})}\cbr{\frac{1}{n}\sum_{i=1}^{n}\varepsilon_i\tilde{g}_i(\hat{f}(x_i)-f(x_i))}.
\]
From Lemma \ref{thm:shepard_interpolation}, for any $v\in\RR^n$ such that $\|v\|_n\le r$, there is a continuous function $q:\cX\rightarrow\RR$ such that $q(x_i)-\hat{f}(x_i)=v_i,\ \forall i=1:n$. By the universal approximation property of $\cF$, $\forall \epsilon>0$, $\exists h\in\cF$ such that $\|h-q\|_{\infty}\le \epsilon r$. Then, we know that $\|h-\hat{f}\|_n\le \|h-q\|_n+\|q-\hat{f}\|_n\le (1+\epsilon)r$. Taking $\epsilon=\min\cbr{\frac{\|\eta\|_n}{2L},1}$ and $\tilde{v}=\argmax_{v\in\RR^n:\|v\|_n\le r}\cbr{\frac{1}{n}\sum_{i=1}^{n}\varepsilon_i\tilde{g}_iv_i}$, by the argument above, we have
\begin{align*}
    r\|\eta\|_n=&\frac{1}{n}\sum_{i=1}^{n}\varepsilon_i\tilde{g}_i\tilde{v}_i=\frac{1}{n}\sum_{i=1}^{n}\varepsilon_i\tilde{g}_i(q(x_i)-\hat{f}(x_i))\\
    =&\frac{1}{n}\sum_{i=1}^{n}\varepsilon_i\tilde{g}_i(h(x_i)-\hat{f}(x_i))+\frac{1}{n}\sum_{i=1}^{n}\varepsilon_i\tilde{g}_i(g(x_i)-h(x_i))\\
    \le&\sup_{f\in\cB_{2r}(\hat{f})}\cbr{\frac{1}{n}\sum_{i=1}^{n}\varepsilon_i\tilde{g}_i(f(x_i)-\hat{f}(x_i))}+\frac{1}{n}\sum_{i=1}^{n}\varepsilon_i\tilde{g}_i(g(x_i)-h(x_i)).
\end{align*}
Therefore, since $\epsilon=\min\cbr{\frac{\|\eta\|_n}{2L},1}$, we have $$\sup_{f\in\cB_{2r}(\hat{f})}\cbr{\frac{1}{n}\sum_{i=1}^{n}\varepsilon_i\tilde{g}_i(f(x_i)-\hat{f}(x_i))}\ge r\|\eta\|_n-\frac{1}{n}\sum_{i=1}^{n}\varepsilon_i\tilde{g}_i(g(x_i)-h(x_i))\ge \frac{r\|\eta\|_n}{2}.$$
In the last inequality, we use the approximation property of $h$.

 By the smoothness of the loss $\cl$, we have $\|\eta-w\|_n\le \beta\|\tilde{f}-f^*\|_n$. On the other hand, since we assume that $\Var(\cl(\tilde{f}(x_i),y_i)|x_i)\ge \tau^2$. Therefore, for any $\delta\in(0,1)$, when $n$ is large enough, with probability at least $1-\delta$, 
 \[
 \|w\|_n^2\ge \tau^2-L\sqrt{\frac{\log(1/\delta)}{n}} \ge \frac{\tau^2}{2}.
 \]
Thus, when $n\ge \frac{4L\log(1/\delta)}{\tau^4}$, with probability at least $1-\delta$, we have $\|\eta-w\|_n\le \frac{4M}{\tau}\|\eta\|_n$.

Combining these parts together, for $n\ge \frac{4L\log(1/\delta)}{\tau^4}$, with probability at least $1-\delta$, we have
$$B_n^\sh(\hat{f})\le \frac{8\beta M}{\tau}W_n(2r).$$ Similarly, we have $B_n^\dia(\hat{f})\le \frac{8\beta M}{\tau}T_n(2r)$. Using Lemma \ref{lemma:bound_W_n-T_n}, we conclude that with probability $1-2\delta$, 
$$B_n^\dia(\hat{f})+B_n^\sh(\hat{f})\le \frac{8\beta M}{\tau}(\Opt^\dia(f^\dia_{\rho_1})+\Opt^\sh(f^\sh_{\rho_2})).$$
   \end{proof}
\subsection{Proofs in Subsection \ref{subsec:bounding_hatr}}
\begin{proof}[Proof of Theorem \ref{thm:bounding_hatr}]
Now, we provide proofs about Theorem \ref{thm:bounding_hatr}. Specifically, we require the following lemmas.
\begin{lemma}\label{lemma:bound_hatr_lemma1}
    For any $t>0$, with probability at least $1-3e^{-t^2}$,
    \[
    \hat{r}_n^2\le \frac{2}{\alpha}\rbr{2U_n^\varepsilon(2\hat{r}_n)+\frac{2\sqrt{2}\sigma(3\sqrt{\log n}+7)\hat{r}_n t}{\sqrt{n}}},
    \]
    where $U_n^\varepsilon(r):=\sup_{f\in\cB_r(\hat{f})}\cbr{\frac{1}{n}\sum_{i=1}^{n}\varepsilon_i\tilde{w}_i(f(x_i)-\hat{f}(x_i))}$ is already defined above.
\end{lemma}
The leading term in Lemma \ref{lemma:bound_hatr_lemma1} is $\frac{8}{\alpha}U_n^\varepsilon(2\hat{r}_n)$ as long as $\cF$ yields a function class that is as complex as a parametric one \citep{wainwright2019high}. Then, utilizing the property of doubly wild-refitting, we have the following lemma.
\begin{lemma}\label{lemma:bound_hatr_lemma2}
For any $t>0$, with probability at least $1-2e^{-t^2}$,
\[
U_n^\varepsilon(2\hat{r}_n)\le\frac{1}{2}\rbr{W_n(2\hat{r}_n))+T_n(2\hat{r}_n)+B_n^\dia(\hat{f})+B_n^\sh(\hat{f})}+\frac{2\sqrt{2}\sigma\hat{r}_nt}{\sqrt{n}},
\]
where $$W_n=\sup_{f\in\cB_r(\hat{f})}\cbr{\frac{1}{n}\sum_{i=1}^{n}\varepsilon_i\tilde{g}_i(f(x_i)-\hat{f}(x_i))},\ T_n=\sup_{f\in\cB_r(\hat{f})}\cbr{\frac{1}{n}\sum_{i=1}^{n}\varepsilon_i\tilde{g}_i(\hat{f}(x_i)-f(x_i))}.$$
\end{lemma}
Combining these two lemmas, we have that
\[
\hat{r}_n^2\le \frac{2}{\alpha}\rbr{W_n(2\hat{r}_n))+T_n(2\hat{r}_n)+B_n^\dia(\hat{f})+B_n^\sh(\hat{f})}+\rbr{\frac{6\sqrt{\log n}+14}{\alpha}+1}\frac{2\sqrt{2}\sigma t}{\sqrt{n}}\hat{r}_n.
\]
This is a quadratic inequality, solving it and we get
\begin{align*}
\hat{r}_n\le& \sqrt{\frac{2}{\alpha}\rbr{W_n(2\hat{r}_n))+T_n(2\hat{r}_n)+B_n^\dia(\hat{f})+B_n^\sh(\hat{f})}}+\rbr{\frac{6\sqrt{\log n}+14}{\alpha}+1}\frac{2\sqrt{2}\sigma t}{\sqrt{n}}
\end{align*}
So we finish the proof.
   \end{proof}
\begin{proof}[Proof of Corollary \ref{cor:bounding_hatr}]
We first show that $W_n$ and $T_n$ are concave functions. We prove the case for $W_n$. For any $s,t$ and $\alpha\in[0,1]$, denote $r:=\alpha s+(1-\alpha)t$. Let $f_s$ and $f_t$ be functions achieving the suprema that define $W_n(s)$ and $W_n(t)$. Define $f_r=\alpha f_s+(1-\alpha)f_t\in\cF$. By the triangle inequality, \[
\|f_r-\hat{f}\|_n\le \alpha\|f_s-\hat{f}\|_n+(1-\alpha)\|f_t-\hat{f}\|_n\le \alpha s+(1-\alpha)t=r.
\]
Thus, the function $f_r$ is feasible for the supremum defining $W_n(r)$, so that we have
\[
\alpha W_n(s)+(1-\alpha)W_n(t)\le W_n(r).
\]
Similarly, $r\mapsto T_n(r)$ is concave. By the fact that $W_n(0)=T_n(0)=0$, we have that
    \[
    \frac{W_n(s)}{s}\le \frac{W_n(t)}{t},\ \frac{T_n(s)}{s}\le \frac{T_n(t)}{t},\ \text{for any}\  s\ge t>0.
    \]
    Now we prove the corollary. We either have 1) $\hat{r}_n\le \max\cbr{r_{\rho_1}^\dia,r_{\rho_2}^\sh}$ or 2) $\hat{r}_n> \max\cbr{r_{\rho_1}^\dia,r_{\rho_2}^\sh}$. In the latter circumstance, we have 
    \[
    W_n(2\hat{r}_n)=2\hat{r}_n\frac{W_n(2\hat{r}_n)}{2\hat{r}_n}\le 2\hat{r}_n\frac{W_n(2r_{\rho_1}^\dia)}{2r_{\rho_1}^\dia}=\hat{r}_n\frac{W_n(2r_{\rho_1}^\dia)}{r_{\rho_1}^\dia}.
    \]
    Similarly, we have
    \[
    T_n(2\hat{r}_n)\le \hat{r}_n\frac{T_n(2r_{\rho_2}^\sh)}{r_{\rho_2}^\sh}.
    \]
    Combining these inequalities with Theorem \ref{thm:bounding_hatr}, we get a quadratic inequality with respect to $\hat{r}_n$.
    \begin{align*}
    \hat{r}_n^2\le \rbr{\frac{4W_n(2r_{\rho_1}^\dia)}{\alpha r_{\rho_1}^\dia}+\frac{4T_n(2r_{\rho_2}^\sh)}{\alpha r_{\rho_2}^\sh}}\cdot\hat{r}_n+\frac{4(B_n^\dia(\hat{f})+B_n^\sh(\hat{f}))}{\alpha}
    +\rbr{\frac{6\sqrt{\log n}+14}{\alpha}+1}^2\frac{16\sigma^2 t^2}{n}.
    \end{align*}
    We thus get our result. Moreover, we could solve this inequality to get 
\[
\hat{r}_n\le \rbr{\frac{4W_n(2r_{\rho_1}^\dia)}{\alpha r_{\rho_1}^\dia}+\frac{4T_n(2r_{\rho_2}^\sh)}{\alpha r_{\rho_2}^\sh}}+\sqrt{\frac{8(B_n^\dia(\hat{f})+B_n^\sh(\hat{f}))}{\alpha}}+\rbr{\frac{6\sqrt{\log n}+14}{\alpha}+1}\frac{4\sqrt{2}\sigma t}{\sqrt{n}}.
\]
Thus, we finish the proof.
   
\section{Proofs in Appendix \ref{app:proofs_sec:statistical_guarantee}}\label{app:proofs_lemmas_in_proofs_sec:statistical_guarantee}
We now provide proofs of the lemmas in Appendix \ref{app:proofs_sec:statistical_guarantee}.
\end{proof}
\begin{proof}[Proof of Lemma \ref{lemma:Opt*<Opt_dagger}]
    First, by some algebra, we have that
    \[
    \Opt^*(\hat{f})=\Opt^\dagger(\hat{f})+\frac{1}{n}\sum_{i=1}^{n}\cl_1'(f^*(x_i),y_i)(f^\dagger(x_i)-f^*(x_i)).
    \]
    We now analyze the term $\frac{1}{n}\sum_{i=1}^{n}\cl_1'(f^*(x_i),y_i)(f^\dagger(x_i)-f^*(x_i))$.
    \[
G(w):w\mapsto \frac{1}{n}\sum_{i=1}^{n}w_i(f^\dagger(x_i)-f^*(x_i)).
\]
$G(w)$ is Lipschitz continuous with respect to the norm $||\cdot||_2$.
\begin{align*}
|G(w)-G(w')|=&|\frac{1}{n}\sum_{i=1}^{n}(w_i-w_i')(f^\dagger(x_i)-f^*(x_i))|\\
\le&\frac{1}{n}\sum_{i=1}^{n}|w_i-w_i'|\cdot|f^\dagger(x_i)-f^*(x_i)|\\
\le&\frac{1}{n}||w-w'||_2(\sum_{i=1}^{n}|f^\dagger(x_i)-f^*(x_i)|^2)^{1/2}\\
=&\frac{||f^\dagger-f^*||_n}{\sqrt{n}}||w-w'||_2.
\end{align*}
We use the concentration inequality about sub-Gaussian random vectors in Lipschitz concentration inequality (Lemma \ref{lemma:lipschitz_concentration}). 
Therefore, with probability at least $1-e^{-t^2}$, 
\[
G(w)\le \sqrt{2}\sigma\frac{||f^\dagger-f^*||_nt}{\sqrt{n}}.
\]
We finish the proof.
   
Now, we focus on the term $\Opt^\dagger(\hat{f})$.
\end{proof}
\begin{proof}[Proof of Lemma \ref{lemma:Opt_dagger<Zn}]
    Throughout the proof of this lemma, we condition on the direction vectors $\cbr{e_i}_{i=1}^{n}$ of $\cbr{w_i}_{i=1}^{n}$ as fixed. We define the function $S(w):=\Opt^\dagger(\hat{f})$, Then, we have that
    \begin{align*}
        S(w)-S(w')=&\frac{1}{n}\sum_{i=1}^{n}(w_i-w_i')(\hat{f}(x_i)-f^\dagger(x_i))\\
        \le&\frac{1}{n}\sum_{i=1}^{n}|w_i-w_i'||\hat{f}(x_i)-f^\dagger(x_i)|\\
        \le&\frac{\|\hat{f}-f^\dagger\|_n}{\sqrt{n}}\|w-w'\|_{2}\\
        \le& \frac{r}{\sqrt{n}}\|w-w'\|_{2}.
    \end{align*}
    In the last inequality, we use the assumption that $r\ge \|\hat{f}-f^\dagger\|_n$. Swapping the order of $w$ and $w'$, we show that $S(w)$ is Lipschitz continuous with constant $\frac{r}{\sqrt{n}}$. We apply Lemma \ref{lemma:lipschitz_concentration} to get that with probability at least $1-e^{-t^2}$,
    \[
    S(w)\le \EE_{w_{1:n}}[\Opt^\dagger(\hat{f})]+\frac{\sqrt{2}\sigma rt}{\sqrt{n}}.
    \]
    By Assumption \ref{ass:loss_function}, we know that $w_i$ are zero-mean along all directions. Therefore, we define $w_i'$ to be an independent copy of $w_i$, i.e., $w_i'\perp w_i$.
    \begin{align*}
        \EE_{w_{1:n}}[\Opt^\dagger(\hat{f})]&\le \EE_{w_{1:n}}\sbr{\sup_{f\in\cB_r(f^\dagger)}\frac{1}{n}\sum_{i=1}^{n}w_i\rbr{f(x_i)-f^\dagger(x_i)}}\\
    &=\EE_{w_{1:n}}\sbr{\sup_{f\in\cB_r(f^\dagger)}\frac{1}{n}\sum_{i=1}^{n}(w_i-\EE[w_i'|e_i])(f(x_i)-f^\dagger(x_i))}\\
    &\le\EE_{w,w'}\sbr{\sup_{f\in\cB_r(f^\dagger)}\frac{1}{n}\sum_{i=1}^{n}(w_i-w_i')(f(x_i)-f^\dagger(x_i))},
    \end{align*}
    where the equality is by Assumption \ref{ass:loss_function} and the second inequality is by Jenson's inequality. 
    Conditioned on $e_i,i=1,\cdots,n$, $w_i-w_i'$ has a symmetric distribution, then, we get $w_i-w_i'\overset{d}{=}2\varepsilon_i\tilde{w}_i$, where $\cbr{\varepsilon_i}_{i=1}^{n}$ is the Rademacher random variable sequence we simulate in Algorithm \ref{alg:wild-refitting}. By the definition of $Z_n^\varepsilon(r)$, we have
    \[
    \EE_{w_{1:n},w'_{1:n}}\sbr{\sup_{f\in\cB_r(f^\dagger)}\frac{1}{n}\sum_{i=1}^{n}(w_i-w_i')(f(x_i)-f^\dagger(x_i))}=2\EE_{\varepsilon_{1:n},\tilde{w}_{1:n}}[Z_n^\varepsilon(r)].
    \]
    Therefore, we prove
    \[
    \EE_{w_{1:n}}[\Opt^\dagger(\hat{f})]\le 2\EE_{\varepsilon_{1:n},\tilde{w}_{1:n}}[Z_n^\varepsilon(r)].
    \]
    Combining this with the concentration inequality of $S(w)$, we finish the proof.
   \end{proof}
\begin{proof}[Proof of Lemma \ref{lemma:Zn<Un}]
    Recall the definition that $$U_n^\varepsilon(r)=\sup_{f\in\cB_r(\hat{f})}\cbr{\frac{1}{n}\sum_{i=1}^{n}\varepsilon_i\tilde{w}_i(f(x_i)-\hat{f}(x_i))},$$ 
    $$Z_n^\varepsilon(r)=\sup_{f\in\cB_r(f^\dagger)}\cbr{\frac{1}{n}\sum_{i=1}^{n}\varepsilon_i\tilde{w}_i(f(x_i)-f^\dagger(x_i))}.$$
    Then, denoting $h$ to be any function that achieves the supremum in $Z_n^\varepsilon(r)$, then we have
    \begin{align*}
        Z_n^\varepsilon(r)=&\frac{1}{n}\sum_{i=1}^{n}(\varepsilon_i\tilde{w}_i)(h(x_i)-f^\dagger(x_i))\\
        =&\frac{1}{n}\sum_{i=1}^{n}\varepsilon_i\tilde{w}_i(h(x_i)-\hat{f}(x_i)+\hat{f}(x_i)-f^\dagger(x_i))\\
        =&\underbrace{\frac{1}{n}\sum_{i=1}^{n}\varepsilon_i\tilde{w}_i(h(x_i)-\hat{f}(x_i))}_{\text{term I}}+\underbrace{\frac{1}{n}\sum_{i=1}^{n}\varepsilon_i\tilde{w}_i(\hat{f}(x_i)-f^\dagger(x_i))}_{\text{term II}}.
    \end{align*}
    We take expectation with respect to $\varepsilon_i,\ i=1:n$ on both terms and analyze them separately.
    For term II, $\varepsilon_i$ is independent of $\hat{f}$ and $f^\dagger$, therefore,
    \[
    \EE_{\varepsilon_{1:n}}[\frac{1}{n}\sum_{i=1}^{n}\varepsilon_i\tilde{w}_i(\hat{f}(x_i)-f^\dagger(x_i))]=0.
    \]
    Then, $\EE_{\varepsilon_{1:n},\tilde{w}_{1:n}}[\frac{1}{n}\sum_{i=1}^{n}\varepsilon_i\tilde{w}_i(\hat{f}(x_i)-f^\dagger(x_i))]=0$.
    
    For term I,since $h\in\cB_r(f^\dagger)$, then given $r\ge\hat{r}_n=\|\hat{f}-f^\dagger\|_n$, by the triangle inequality, we know that
    \[
    \|h-\hat{f}\|_n\le r+\hat{r}_n=:r_1.
    \]
    Thus,
    \[
    \text{term I}\le \sup_{f\in\cB_{r_1}(\hat{f})}\frac{1}{n}\sum_{i=1}^{n}\varepsilon_i\tilde{w}_i(f(x_i)-\hat{f}(x_i))=U_n^\varepsilon(r+\hat{r}_n).
    \]
    Taking expectation with respect to $\varepsilon_{1:n}$ and $\tilde{w}_{1:n}$, we finish the proof.
   \end{proof}
\begin{proof}[Proof of Lemma \ref{lemma:bound_E[Un]}]
    For the second claim, we first condition on $\tilde{w}_{1:n}$ and view $U_n^\varepsilon(r)$ as a function of $\varepsilon$. We define function
    \[
    Y(\varepsilon):=\sup_{f\in\cB_r(\hat{f})}\cbr{\frac{1}{n}\sum_{i=1}^{n}\varepsilon_i\tilde{w}_i(f(x_i)-\hat{f}(x_i))}
    \]
    For $r\ge \hat{r}_n$, we have that
    \begin{align*}
        Y(\varepsilon)-Y(\varepsilon')&=\sup_{f\in\cB_r(\hat{f})}\cbr{\frac{1}{n}\sum_{i=1}^{n}\varepsilon_i\tilde{w}_i(f(x_i)-\hat{f}(x_i))}-\sup_{f\in\cB_r(\hat{f})}\cbr{\frac{1}{n}\sum_{i=1}^{n}\varepsilon_i'\tilde{w}_i(f(x_i)-\hat{f}(x_i))}\\
        &\le \sup_{f\in\cB_r(\hat{f})}\cbr{\frac{1}{n}\sum_{i=1}^{n}(\varepsilon_i-\varepsilon_i')\tilde{w}_i(f(x_i)-\hat{f}(x_i))}\\
        &\le \sup_{f\in\cB_r(\hat{f})}\cbr{\frac{1}{n}\sum_{i=1}^{n}|(\varepsilon_i-\varepsilon_i')\tilde{w}_i|\cdot|f(x_i)-\hat{f}(x_i)|}\\
        &\le \frac{(\max_{i=1,\cdots,n}\|\tilde{w}_i\|_2)r}{\sqrt{n}}\|\varepsilon-\varepsilon'\|_2.
    \end{align*}
    Similarly, $Y(\varepsilon')-Y(\varepsilon)$ has the same upper bound, so we conclude that $Y(\varepsilon)$ is Lipschitz-continuous with constant $\frac{(\max_{i=1,\cdots,n}\|\tilde{w}_i\|_2)rt}{\sqrt{n}}$. Therefore, applying Lemma \ref{lemma:lipschitz_concentration}, we have that with probability at least $1-e^{-t^2}$,
    \[
    \EE_{\varepsilon_{1:n}}[U_n^\varepsilon(r)]\le U_n^\varepsilon(r)+\frac{\sqrt{2}(\max_{i=1,\cdots,n}|\tilde{w}_i|_2)rt}{\sqrt{n}}.
    \]
    Taking expectation on $\tilde{w}_{1:n}$, we have that
    \[
    \EE_{\varepsilon_{1:n},\tilde{w}_{1:n}}[U_n^\varepsilon(r)]\le \EE_{\tilde{w}_{1:n}}[U_n^\varepsilon(r)]+\frac{\sqrt{2}\EE_{\tilde{w}_{1:n}}[\max_{i=1,\cdots,n}|\tilde{w}_i|]rt}{\sqrt{n}}.
    \]
    We focus on the right hand side of this inequality. For the term $\EE_{\tilde{w}_{1:n}}[\max_{i\in\cbr{1,\cdots,n}}|\tilde{w}_i|]$, we apply Lemma \ref{lemma:sub_gaussian_concentration} and get
    \[
    \EE_{\tilde{w}_{1:n}}[\max_{i\in\cbr{1,\cdots,n}}|\tilde{w}_i|]\le \sigma\sqrt{8(\log 5+\log 2n)}.
    \]
    For the term $\EE_{\tilde{w}_{1:n}}[U_n^\varepsilon(r)]$, we define the function 
    $$Q(\tilde{w}):=\sup_{f\in\cB_r(\hat{f})}\cbr{\frac{1}{n}\sum_{i=1}^{n}\varepsilon_i\tilde{w}_i(f(x_i)-\hat{f}(x_i))}.$$
    Then, we have that
    \begin{align*}
        Q(\tilde{w})-Q(\tilde{w}')\le\sup_{f\in\cB_r(\hat{f})}\cbr{\frac{1}{n}\sum_{i=1}^{n}|\tilde{w}-\tilde{w}'|\cdot|f(x_i)-\hat{f}(x_i)|}\le  \frac{r}{\sqrt{n}}\|\tilde{w}-\tilde{w}'\|_2.
    \end{align*}
    We then apply Lemma \ref{lemma:lipschitz_concentration}, to get that with probability at least $1-e^{-t^2}$,
    \[
    \EE_{\tilde{w}_{1:n}|}[U_n^\varepsilon(r)]\le U_n^\varepsilon(r)+\frac{\sqrt{2}\sigma rt}{\sqrt{n}}.
    \]
    Combining all these parts together, we have that with probability at least $1-2e^{-t^2}$,
    \[
    \EE_{\varepsilon_{1:n},\tilde{w}_{1:n}|}[U_n^\varepsilon(r)]\le U_n^\varepsilon(r)+\frac{\sqrt{2}\sigma(5+3\sqrt{\log n}+1)rt}{\sqrt{n}}.
    \]
    So we finish the proof.
   \end{proof}
    
\begin{proof}[Proof of Lemma \ref{lemma:Un<wild_opt}]
    Recall that $w_i'$ is an independent copy of $w_i=\cl_1'(f^*(x_i),y_i)$; then we have
\begin{align*}
    U_n^\varepsilon(r)=&\sup_{f\in\cB_r(\hat{f})}\cbr{\frac{1}{n}\sum_{i=1}^{n}\varepsilon_i\frac{w_i-w_i'}{2}(f(x_i)-\hat{f}(x_i))}\\
    \le&\underbrace{\frac{1}{2}\sup_{f\in\cB_r(\hat{f})}\cbr{\frac{1}{n}\sum_{i=1}^{n}\varepsilon_i w_i(f(x_i)-\hat{f}(x_i))}}_{\text{term I}}+\underbrace{\frac{1}{2}\sup_{f\in\cB_r(\hat{f})}\cbr{\frac{1}{n}\sum_{i=1}^{n}\varepsilon_i w_i'(\hat{f}(x_i)-f(x_i))}}_{\text{term II}}.\\
\end{align*}
For the first term, we could directly use the property of $f_{\rho_1}^\dia$ to get:
\[
\sup_{f\in\cB_r(\hat{f})}\cbr{\frac{1}{n}\sum_{i=1}^{n}\varepsilon_iw_i(f(x_i)-\hat{f}(x_i))}\le \Opt^\dia_{\rho_1}(f_{\rho_1}^\dia)+B_n^\dia(\hat{f}).
\]
For term II, it is trickier because $w_i'$ is not the original noise, but rather a conditionally independent copy whose realization is not contained in the dataset. We first apply Lemma \ref{lemma:lipschitz_concentration} to obtain that, with probability at least $1-e^{-t^2}$,
\[
\sup_{f\in\cB_r(\hat{f})}\cbr{\frac{1}{n}\sum_{i=1}^{n}\varepsilon_i w_i'(\hat{f}(x_i)-f(x_i))}\le \EE_{\varepsilon, w_i'}\sbr{\sup_{f\in\cB_r(\hat{f})}\frac{1}{n}\sum_{i=1}^{n}\varepsilon_i w_i'(\hat{f}(x_i)-f(x_i))}+\frac{\sqrt{2}r\sigma t}{\sqrt{n}}.
\]
By the definition of independent copy, we have
\begin{align*}
&\EE_{\varepsilon, w'}\sbr{\sup_{f\in\cB_r(\hat{f})}\frac{1}{n}\sum_{i=1}^{n}\varepsilon_i w_i'(\hat{f}(x_i)-f(x_i))}+\frac{\sqrt{2}r\sigma t}{\sqrt{n}}\\
=&\EE_{\varepsilon, w}\sbr{\sup_{f\in\cB_r(\hat{f})}\frac{1}{n}\sum_{i=1}^{n}\varepsilon_i w_i(\hat{f}(x_i)-f(x_i))}+\frac{\sqrt{2}r\sigma t}{\sqrt{n}}
\end{align*}
Again, by applying Lemma \ref{lemma:lipschitz_concentration}, we have that with probability at least $1-e^{-t^2}$,
\[
\EE_{\varepsilon, w}\sbr{\sup_{f\in\cB_r(\hat{f})}\frac{1}{n}\sum_{i=1}^{n}\varepsilon_i w_i(\hat{f}(x_i)-f(x_i))}\le \sup_{f\in\cB_r(\hat{f})}\frac{1}{n}\sum_{i=1}^{n}\varepsilon_i w_i(\hat{f}(x_i)-f(x_i))+\frac{\sqrt{2}r\sigma t}{\sqrt{n}}.
\]
Thus, we have that with probability at least $1-2e^{-t^2}$, for any $r>0$,
\begin{align*}
U_n^\varepsilon(r)\le &\frac{1}{2}\rbr{\sup_{f\in\cB_r(\hat{f})}\{\frac{1}{n}\sum_{i=1}^{n}\varepsilon_i w_i(f(x_i)-\hat{f}(x_i))\}+\sup_{f\in\cB_r(\hat{f})}\{\frac{1}{n}\sum_{i=1}^{n}\varepsilon_i w_i(\hat{f}(x_i)-f(x_i))\}}\\
&+\frac{\sqrt{2}\sigma rt}{\sqrt{n}}.
\end{align*}
By the doubly wild refitting procedure, we have
    \[
    \varepsilon_i\tilde{g}_i=\varepsilon_i\cl_1'(\tilde{f}(x_i),y_i)=\varepsilon_i\rbr{\cl_1'(\tilde{f}(x_i),y_i)-\cl_1'(f^*(x_i),y_i)}+\varepsilon_i\cl_1'(f^*(x_i),y_i).
    \]
    This is equivalent to
    \[
    \varepsilon_i w_i=\varepsilon_i\tilde{g}_i+\varepsilon_i\rbr{\cl_1'(f^*(x_i),y_i)-\cl_1'(\tilde{f}(x_i),y_i)}.
    \]
    Plugging this in, we have 
    \begin{align*}
        &\sup_{f\in\cB_r(\hat{f})}\cbr{\frac{1}{n}\sum_{i=1}^{n}\varepsilon_i w_i(f(x_i)-\hat{f}(x_i))}+\sup_{f\in\cB_r(\hat{f})}\cbr{\frac{1}{n}\sum_{i=1}^{n}\varepsilon_i w_i(\hat{f}(x_i)-f(x_i))}\\
        \le&\sup_{f\in\cB_r(\hat{f})}\cbr{\frac{1}{n}\sum_{i=1}^{n}\varepsilon_i \tilde{g}_i(f(x_i)-\hat{f}(x_i))}+\sup_{f\in\cB_r(\hat{f})}\cbr{\frac{1}{n}\sum_{i=1}^{n}\varepsilon_i \tilde{g}_i(\hat{f}(x_i)-f(x_i))}\\
        +&\sup_{f\in\cB_r(\hat{f})}\cbr{\frac{1}{n}\sum_{i=1}^{n}\varepsilon_i\rbr{\cl_1'(f^*(x_i),y_i)-\cl_1'(\tilde{f}(x_i),y_i)}(f(x_i)-\hat{f}(x_i))}\\
        +&\sup_{f\in\cB_r(\hat{f})}\cbr{\frac{1}{n}\sum_{i=1}^{n}\varepsilon_i\rbr{\cl_1'(f^*(x_i),y_i)-\cl_1'(\tilde{f}(x_i),y_i)}(\hat{f}(x_i)-f(x_i))}\\
    \end{align*}
By the condition that $\|f^\dia_{\rho_1}-\hat{f}\|_n=\|f^\sh_{\rho_2}-\hat{f}\|_n=2r$, we apply Lemma \ref{lemma:bound_W_n-T_n} and set $r\leftarrow2r$ to have:
    \[
    \sup_{f\in\cB_{2r}(\hat{f})}\cbr{\frac{1}{n}\sum_{i=1}^{n}\varepsilon_i \tilde{g}_i(f(x_i)-\hat{f}(x_i))}=W_n(\|f^\dia_{\rho_1}-\hat{f}\|_n)\le \Opt^\dia(f^\dia_{\rho_1});
    \]
    \[
    \sup_{f\in\cB_{2r}(\hat{f})}\cbr{\frac{1}{n}\sum_{i=1}^{n}\varepsilon_i \tilde{g}_i(\hat{f}(x_i)-f(x_i))}=T_n(\|f^\sh_{\rho_2}-\hat{f}\|_n)\le \Opt^\sh(f^\sh_{\rho_2});
    \]
    \[
    B_n^\dia(\hat{f}):=\sup_{f\in\cB_{2r}(\hat{f})}\cbr{\frac{1}{n}\sum_{i=1}^{n}\varepsilon_i\rbr{\cl_1'(f^*(x_i),y_i)-\cl_1'(\tilde{f}(x_i),y_i)}(f(x_i)-\hat{f}(x_i))};
    \]
    \[
    B_n^\sh(\hat{f}):=\sup_{f\in\cB_{2r}(\hat{f})}\cbr{\frac{1}{n}\sum_{i=1}^{n}\varepsilon_i\rbr{\cl_1'(f^*(x_i),y_i)-\cl_1'(\tilde{f}(x_i),y_i)}(\hat{f}(x_i)-f(x_i))}.
    \]
    Adding them together, we prove that when $r\ge\hat{r}_n$, and the noise scales $\rho_1$, $\rho_2$ satisfy $\|f^\dia_{\rho_1}-\hat{f}\|_n=\|f^\sh_{\rho_2}-\hat{f}\|_n=2r$, with probability at least $1-2e^{-t^2}$,
    \begin{align}\label{ineq:bounding_Un}
    U_n^\varepsilon(2r)\le &\frac{1}{2}\rbr{W_n(2r)+T_n(2r)+B_n^\dia(\hat{f})+B_n^\sh(\hat{f})}+\frac{2\sqrt{2}\sigma rt}{\sqrt{n}}\nonumber\\
    \le& \frac{1}{2}\rbr{\Opt^\dia(f^\dia_{\rho_1})+\Opt^\sh(f^\sh_{\rho_2})+B_n^\dia(\hat{f})+B_n^\sh(\hat{f})}+\frac{2\sqrt{2}\sigma rt}{\sqrt{n}},
    \end{align}
where $\Opt^\dia(f^\dia_{\rho_1}),\ \Opt^\sh(f^\sh_{\rho_2})$ are wild optimisms and $B_n^\dia(\hat{f}),\ B_n^\sh(\hat{f})$ are pilot error terms. Thus, we finish the proof.    
   \end{proof}
\begin{proof}[Proof of Lemma \ref{lemma:bound_hatr_lemma1}]
Recall that
$f^\dagger\in\argmin_{f\in\cF}\cbr{\frac{1}{n}\sum_{i=1}^{n}\EE_{y_i}[\cl(f(x_i),y_i)]}$. Then, we compute the Gâteaux derivative and apply the first order optimality condition to get that
\[
\frac{1}{n}\sum_{i=1}^{n}\EE_{y_i}[\cl_1'(f^\dagger(x_i),y_i)](f(x_i)-f^\dagger(x_i))=0.
\]
By Assumption \ref{ass:loss_function} and the well-specified condition $f^\dagger=f^*$, we have that
\begin{align*}
    \|\hat{f}-f^\dagger\|_n^2\le&\frac{2}{\alpha}\frac{1}{n}\sum_{i=1}^{n}D_{\cl(\cdot,y_i)}(\hat{f}(x_i),f^\dagger(x_i))\\
    =&\frac{2}{\alpha n}\sum_{i=1}^{n}\cl(\hat{f}(x_i),y_i)-\cl(f^\dagger(x_i),y_i)(-\cl_1'(f^\dagger(x_i),y_i))(\hat{f}(x_i)-f^\dagger(x_i))\\
    \le &\frac{2}{\alpha n}\sum_{i=1}^{n}-\cl_1'(f^\dagger(x_i),y_i)(\hat{f}(x_i)-f^\dagger(x_i))\\
    =&\frac{2}{\alpha n}\sum_{i=1}^{n}(-w_i)(\hat{f}(x_i)-f^\dagger(x_i)).
\end{align*}
In the last inequality, we use the property that ERM $\hat{f}$ minimizes the empirical risk. In the last equality, we apply the well-specification condition that $f^\dagger=f^*$. 

Define the empirical process $M(w):=\sup_{f\in\cB_{\hat{r}_n}(f^\dagger)}\cbr{\frac{1}{n}\sum_{i=1}^{n}(-w_i)(\hat{f}(x_i)-f^\dagger(x_i))}$.
Then, we apply the Lemma \ref{lemma:lipschitz_concentration} to get that for any $t>0$, with probability at least $1-e^{-t^2}$,
\[
M(w)\le \EE_{w_{1:n}}[M(w)]+\frac{\sqrt{2}\sigma \hat{r}_n t}{\sqrt{n}}.
\]
Similar to the proof of Lemma \ref{lemma:Opt_dagger<Zn}, we have that
\[
\EE_{w_{1:n}}[M(w)]\le 2\EE_{w_{1:n},\tilde{w}_{1:n}}[Z_n^\varepsilon(r)].
\]
By Lemma \ref{lemma:Zn<Un} and Lemma \ref{lemma:bound_E[Un]}, we have that with probability at least $1-2e^{-t^2}$,
\[
\EE_{w_{1:n},\tilde{w}_{1:n}}[Z_n^\varepsilon(r)]\le U_n^\varepsilon(2\hat{r}_n)+\frac{\sqrt{2}\sigma(5+3\sqrt{\log n}+1)\hat{r}_n t}{\sqrt{n}}.
\]
Combining these two parts together, we have that
\[
M(w)\le 2 U_n^\varepsilon(2\hat{r}_n)+\frac{2\sqrt{2}\sigma(3\sqrt{\log n}+7)\hat{r}_n t}{\sqrt{n}},
\]
which implies
\[
\hat{r}_n^2\le \frac{2}{\alpha}\rbr{2U_n^\varepsilon(2\hat{r}_n)+\frac{2\sqrt{2}\sigma(3\sqrt{\log n}+7)\hat{r}_n t}{\sqrt{n}}}.
\]
This is a quadratic inequality. Solving the inequality, we have that
\begin{align*}
    \hat{r}_n^2\le& \frac{8}{\alpha}U_n^\varepsilon(2\hat{r}_n)+\frac{32\sigma^2(3\sqrt{\log n}+7)^2t^2}{\alpha^2 n}.\\
\end{align*}
So we finish the proof.
   \end{proof}
\begin{proof}[Proof of Lemma \ref{lemma:bound_hatr_lemma2}]
    The proof is the same as proving Lemma \ref{lemma:Un<wild_opt}. By inequality \ref{ineq:bounding_Un}, we have that 
    \[
    U_n^\varepsilon(2r)\le \frac{1}{2}\rbr{W_n(2r)+T_n(2r)+B_n^\dia(\hat{f})+B_n^\sh(\hat{f})}+\frac{2\sqrt{2}\sigma rt}{\sqrt{n}}.
    \]
    This is exactly our claim.
\end{proof}

\section{Computation}\label{app:computation}
In this section, we provide an efficient method for computing $y_i^\dia$ and $y_i^\sh$ in our doubly wild refitting Algorithm \ref{alg:wild-refitting}. The method is called the proximal point algorithm (PPA). 

In general, the proximal point algorithm (PPA) aims to solve the monotone inclusion problem $0\in T(x)$ for maximal monotone operator $T$. By the result of \citet{minty1962monotone}, the Moreau–Yosida resolvent $P_{\lambda}=(I+\lambda T)^{-1}$ will be firmly non-expansive. This suggests that the inclusion problem $0\in T(x)$ can be iteratively approximated via the following recursion $z_{k+1}\approx P_{c_k}(z_k)$, where $c_k$ is the step size or learning rate. We write "$\approx$"  because computing "$z_{k+1}=P_{c_k}(z_k)$" can be computationally intractable. The following theorem of \citet{rockafellar1976monotone} provides the convergence guarantees for this inexact proximal point algorithm.
\begin{theorem}
    Define the operator $S_k(z)=T_k(z)+\frac{1}{c_k}(z-z_k)$. Assuming that computationally, we have the following approximation procedure $z_{k+1}\approx P_k(z_k)=(I+c_k T)^{-1}(z_k)$ that satisfies 
    $$\text{dist}(0,S_k(z_{k+1})=\|S_k(z_{k+1})\|_2\le \frac{\delta_k}{c_k}\|z_{k+1}-z_k\|_2,$$
    for some convergent series $\sum_{i=1}^{\infty}\delta_k<\infty$, then, we have that
    \[
    \|z_{k+1}-P_{c_k}(z_k)\|_2\le \delta_k\|z_{k+1}-z_k\|_2. 
    \]
    The limit point $z_{\infty}=\lim_{k\rightarrow\infty}z_k$ satisfies that $T(z_\infty)=0$. Moreover, if $T^{-1}$ is Lipschitz continuous with constant $a$ in a neighborhood around $0$ with a radius $\tau>0$. Then, for $\cbr{c_k}_{k=1}^{\infty}$ nondecreasing, define $\alpha_k=\frac{a}{\sqrt{a^2+c_k^2}}<1$. There $\exists$ $\bar{k}$ such that for all $k>\bar{k}$,
    \[
    \|z_{k+1}-z_{\infty}\|_2\le \frac{\alpha_k+\delta_k}{1-\delta_k}\| z_{k}-z_{\infty}\|_2,\ \frac{\alpha_k+\delta_k}{1-\delta_k}<1,\ \forall k\ge\bar{k}.
    \]
\end{theorem}
 Therefore, it suffices to focus on computing the iterative update $z_{k+1}\approx P_k(z_k)$. A variety of numerical methods have been developed for evaluating such proximal-type operators, including Douglas–Rachford splitting \citep{eckstein1992douglas}, Peaceman–Rachford splitting \citep{lions1979splitting}, and the Alternating Direction Method of Multipliers (ADMM) \citep{boct2019admm} concerning monotone operators. We do not elaborate on these algorithms here; instead, we refer the reader to the cited works for detailed discussions.

\end{document}